\documentclass{article}
\usepackage[preprint]{neurips_2026}

\usepackage{amsmath,amsfonts,bm}

\def\eqref#1{equation~\ref{#1}}

\def\1{\bm{1}}

\DeclareMathAlphabet{\mathsfit}{\encodingdefault}{\sfdefault}{m}{sl}
\SetMathAlphabet{\mathsfit}{bold}{\encodingdefault}{\sfdefault}{bx}{n}

\usepackage{xurl}
\usepackage{graphicx}
\usepackage{booktabs}
\usepackage{multirow}
\usepackage{makecell}

\usepackage{subcaption}

\usepackage{algorithm}
\usepackage{algpseudocode}
\usepackage{float}

\usepackage{tcolorbox}
\tcbuselibrary{breakable}
\usepackage{placeins}

\usepackage{amsmath}
\usepackage{amsthm}
\usepackage{hyperref}
\usepackage{url}

\title{Your Language Model is Its Own Critic: \\Reinforcement Learning with Value \\Estimation
from Actor's Internal States}

\author{%
  Yunho Choi\textsuperscript{$*$} \\
  Graduate School of Data Science \\
  Seoul National University \\
  \texttt{dbsgh7177@snu.ac.kr}
  \And
  Jongwon Lim\textsuperscript{$*$} \\
  Graduate School of Data Science \\
  Seoul National University \\
  \texttt{elijah0430@snu.ac.kr}
  \AND
  Woojin Ahn \\
  Computer Science and Engineering  \\
  Seoul National University \\
  \texttt{awj1204@snu.ac.kr}
  \And
  Minjae Oh \\
  Graduate School of Data Science \\
  Seoul National University \\
  \texttt{kosair@snu.ac.kr}
  \AND
  Jeonghoon Shim \\
  Graduate School of Data Science \\
  Seoul National University \\
  \texttt{jhshim98@snu.ac.kr}
  \And
  Yohan Jo\textsuperscript{$\dagger$} \\
  Graduate School of Data Science \\
  Seoul National University \\
  \texttt{yohan.jo@snu.ac.kr}
}

\begin{document}
\newtheorem{proposition}{Proposition}
\newtheorem{corollary}{Corollary}
\addtocontents{toc}{\protect\setcounter{tocdepth}{-1}}

\maketitle

\begin{abstract}
Reinforcement learning with verifiable rewards (RLVR) for Large Reasoning Models rests on variance reduction, which requires both a reliable baseline and high prompt diversity within each training batch. This is especially difficult in multi-domain training for general reasoning models, where prompts from different tasks induce highly diverse gradient signals.
Existing approaches fall short in different ways: GRPO estimates its baseline as the group mean over rollouts from the same prompt, so an accurate baseline leaves fewer distinct prompts in the batch, while PPO avoids this trade-off by training a policy scale critic, roughly doubling the cost of training.
We introduce \textbf{POISE} (\underline{P}olicy \underline{O}ptimization with \underline{I}nternal \underline{S}tate Value \underline{E}stimation), a reinforcement learning algorithm that turns the model's internal states into a value model. A lightweight probe reads the signals already computed during the forward pass to predict the baseline, and is trained online alongside the policy.
To preserve gradient unbiasedness, we introduce a cross-rollout construction that predicts each rollout's value from an independent rollout's internal states.
On Qwen3-4B and OLMo3-7B-Instruct-DPO across a six-domain verifiable-reward corpus, POISE outperforms other RLVR baselines while achieving more stable training.
Moreover, the probe matches a separate LLM-scale value model, generalizes to various tasks, and remains accurate as the policy scales.
By leveraging the model's internal representations, POISE enables stable policy optimization. \footnote{Code: \url{https://github.com/holi-lab/poise}}
\end{abstract}

\def\thefootnote{\fnsymbol{footnote}}
\footnotetext[1]{Equal contribution.}
\footnotetext[2]{Corresponding author.}
\def\thefootnote{\arabic{footnote}}

\section{Introduction}

Reinforcement learning with verifiable rewards (RLVR) has become a key driver of progress in large language model (LLM) reasoning~\citep{deepseek_r1, lambert2025tulu3pushingfrontiers}, extending from basic mathematical tasks to recent multi-domain post-training pipelines~\citep{olmo2026olmo3, qwen3, lambert2025tulu3pushingfrontiers, cheng2025guru}. However, RL training is notoriously unstable, and this instability becomes more pronounced in multi-task settings~\citep{teh2017distral}. Thus, enabling stable RL across multiple domains has become increasingly important.


One common way modern RL algorithms address this instability is through advantage estimation: rather than using the sampled reward directly, they subtract a baseline to obtain a relative measure of performance and thus reducing the gradient variance. Obtaining a reliable baseline, however, is nontrivial. GRPO and its variants, among the most widely used RLVR algorithms, estimate the baseline by averaging rewards over multiple rollouts from the same prompt~\citep{shao2024deepseekmath, yu2025dapoopensourcellmreinforcement}. Accurate baseline estimation requires a large group size, but under a fixed rollout budget, this leaves fewer distinct prompts in each training batch. Thus, improving baseline accuracy comes at the cost of prompt diversity, another key determinant of training stability. We demonstrate this trade-off both theoretically, by decomposing policy gradient covariance into within-prompt and between-prompt components (\S~\ref{sec:gradientvariance}), and empirically (Figure~\ref{fig:adam_grad_variance}). Taken together, these results reveal a fundamental flaw in GRPO: the same mechanism that improves baseline estimation reduces prompt diversity, potentially increasing training instability.

Alternative approaches that avoid group sampling exist, but each brings limitations of its own. PPO~\citep{schulman2017proximal} learns the baseline with a critic, but the critic is a LLM-scale model that processes every generated sequence, roughly doubling memory consumption and optimization cost. 
A second line instead relies on knowing how informative each prompt is from previous rollout outcomes~\citep{zheng2025actpaysefficientreinforcement, xu2025singlestreampolicyoptimization}, and estimating it for every prompt and every model to be trained is prohibitively expensive when the training corpus is large. 
This raises the central question of our work: \textit{can we obtain a reliable baseline for any domain, using only minimal rollouts and no separate value model?}


Our answer starts from a simple observation: the signal such a baseline needs may already be inside the policy model. A growing body of work shows that the hidden states of LLMs and LRMs encode outcome-relevant information, including how difficult the model perceives a problem to be~\citep{zhang2025reasoning}, whether it falls within the model's capabilities~\citep{zhu-etal-2025-llm}, and whether an answer is correct~\citep{cencerrado2026no}.
However, their use as a core component of the training algorithm itself remains underexplored.
In this paper, we propose \textbf{POISE} (\underline{P}olicy \underline{O}ptimization with \underline{I}nternal \underline{S}tate Value \underline{E}stimation), a reinforcement learning algorithm that turns the policy model's internal states into a value model. A lightweight probe reads the policy's hidden states from before and after generation along with the uncertainty of the tokens it produced, and predicts the expected reward on that prompt, which serves as the baseline. Additionally, because a baseline read from a response would bias the gradient for that same response~\citep{williams1992simple, tucker2018mirageactiondependentbaselinesreinforcement}, we adopt a \textit{cross-rollout} construction, sampling two independent rollouts per prompt and using each rollout's baseline predicted from the other.

In a six-domain verifiable-reward training corpus~\citep{cheng2025guru}, POISE outperforms DAPO, a state-of-the-art RLVR algorithm, while achieving lower gradient variance (Figure~\ref{fig:gradvarpolicy}) and maintaining stable value estimation throughout training (Figure~\ref{fig:valacc}). Thorough ablations on both value estimation and downstream policy optimization trace these gains to each of our design choices. A critic must scale with the policy while the probe does not, so our approach becomes more valuable as models get larger, precisely where RLVR is most expensive (Table~\ref{tab:generalization-full}). By bringing internal states into policy optimization itself, POISE attains a reliable baseline with high prompt diversity, enabling stable training.



\section{Background}
\label{sec:preliminaries}

\subsection{Preliminaries}

\paragraph{Policy gradient and baseline estimation.}
We formulate RLVR for LLM reasoning as a contextual bandit problem over prompt--response pairs~\citep{wang2026spposequencelevelppolonghorizon}. Given a prompt $x \sim \mathcal{D}$ and a response $y \sim \pi_\theta(\cdot \mid x)$ sampled from the policy model, the objective is to maximize the expected verifiable reward $R(x, y)$,
\begin{equation}
    J(\theta)
    =
    \mathbb{E}_{x \sim \mathcal{D},\, y \sim \pi_\theta(\cdot \mid x)}
    \left[
    R(x,y)
    \right].
\end{equation}
By the policy gradient theorem, 
\begin{equation}
    \nabla_\theta J(\theta)
    =
    \mathbb{E}_{x \sim \mathcal{D},\, y \sim \pi_\theta(\cdot \mid x)}
    \left[
    R(x,y)\, \nabla_\theta \log \pi_\theta(y \mid x)
    \right],
\end{equation}
which yields the REINFORCE estimator~\citep{williams1992simple}. In practice, this estimator is typically combined with a baseline $b(x)$ to reduce variance~\citep{sutton2000policy}, giving the advantage $A(x,y) = R(x,y) - b(x)$ and the gradient estimator
\begin{equation}
    \nabla_\theta J(\theta)
    =
    \mathbb{E}   _{x \sim \mathcal{D},\, y \sim \pi_\theta(\cdot \mid x)}
    \left[
    \left(R(x,y) - b(x)\right)
    \nabla_\theta \log \pi_\theta(y \mid x)
    \right].
    \label{eq:baseline_pg}
\end{equation}
The standard near-optimal choice for variance reduction is the value function~\citep{weaver2001optimal,greensmith2004variance}
\begin{equation}
    V^{\pi_\theta}(x)
    =
    \mathbb{E}_{y \sim \pi_\theta(\cdot \mid x)}
    \left[
    R(x,y)
    \right],
\end{equation}
which is unknown and must be estimated in practice. PPO approximates $V^{\pi_\theta}(x)$ with a learned critic $v_\phi$ that is trained jointly with the policy. GRPO instead samples a group of $G$ responses $\{y^{(1)},\ldots,y^{(G)}\}$ for the same prompt and uses their mean reward as an empirical prompt-level baseline, $b_{\mathrm{GRPO}}(x)    =
    \frac{1}{G}
    \sum_{j=1}^{G} R(x,y^{(j)}),$
obtaining it directly from on-policy rollouts.

\paragraph{Unbiasedness condition for baselines.}
Subtracting a baseline preserves the unbiasedness of the policy gradient only when the baseline term has zero expectation:
\begin{equation}
    \mathbb{E}_{y \sim \pi_\theta(\cdot \mid x)}
    \left[
    b(x)\, \nabla_\theta \log \pi_\theta(y \mid x)
    \right]
    = 0.
    \label{eq:unbiased_baseline}
\end{equation}
This condition holds when the baseline is conditionally independent of the sampled response $y$ given the prompt $x$, in which case
\begin{equation}
    \mathbb{E}_{y \sim \pi_\theta(\cdot \mid x)}
    \left[
    b(x)\, \nabla_\theta \log \pi_\theta(y \mid x)
    \right]
    =
    b(x)\, \nabla_\theta
    \sum_y \pi_\theta(y \mid x)
    = 0.
\end{equation}
Equivalently, a baseline may depend only on the prompt or on any quantity that is independent of the sampled response given the prompt. Violating this condition biases the gradient and can drive the policy to converge suboptimally. We therefore adopt a \emph{cross-rollout} construction (\S \ref{sec:cross_rollout}), where the baseline for a response is computed from another independent response, preserving Eq.~\ref{eq:unbiased_baseline}.

\subsection{Motivation: Gradient Variance and Number of Prompts in the Batch}
\label{sec:motivation}

In this section, we show theoretically and empirically that, under a fixed rollout budget, using more rollouts per prompt reduces prompt diversity and increases policy-gradient variance. This trade-off naturally motivates a reliable baseline estimator that requires minimal per-prompt sampling.

\paragraph{Theoretical analysis.}
\label{sec:gradientvariance}

Let $B$ be the total number of completions in a training batch, with $n$ distinct prompts and $m$ completions each, so $B = n\cdot m$. For prompt $x^{(i)} \sim D$ and completion $y^{(ij)} \sim \pi_\theta(\cdot \mid x^{(i)})$, define the per-sample gradient:
\begin{equation}
    Z(x,y) = \nabla_\theta \log \pi_\theta(y \mid x)\bigl(R(x,y) - b(x)\bigr)
\end{equation}
where $b(x)$ is a baseline. The batch gradient estimator is:
\begin{equation}
   \hat{g} = \frac{1}{n}\sum_{i=1}^n \frac{1}{m}\sum_{j=1}^m Z(x^{(i)}, y^{(ij)}).
   \label{eq:gradient_estimator}
\end{equation}
\begin{proposition}[Gradient variance decomposition]\label{proposition_gradient_variance}
    Let $\Sigma_w$ and $\Sigma_b$ denote the within-prompt and between-prompt covariance matrices of $Z$. Both $\Sigma_w$ and $\Sigma_b$ are fixed properties of $(D, \pi_\theta, R)$, independent of the allocation $(n, m)$. Then:
    \begin{equation}
        \operatorname{Cov}(\hat{g}) = \frac{1}{B}\Sigma_w + \frac{m}{B}\Sigma_b.
    \end{equation}
\end{proposition}

\begin{proof}
See \S~\ref{proof:prop1}.
\end{proof}
\begin{corollary}[Optimal allocation]
    For a fixed budget $B$ and baseline $b(x)$, the variance of $\hat{g}$ is monotonically non-decreasing in $m$ (in the Loewner order) and is minimized at $m=1$, $n=B$.
\end{corollary}
\begin{proof}
  See \S~\ref{proof:cor1}.
\end{proof}

\paragraph{Empirical validation.}
We next test whether the allocation effect predicted by Proposition~\ref{proposition_gradient_variance} arises in practice. At each policy checkpoint, we fix the rollout budget $B$ and compare $m=1$ (one rollout from each of $B$ prompts) with $m=8$ (eight rollouts from each of $B/8$ prompts). Thus, the total number of rollouts remains fixed while prompt diversity changes. Both conditions use fixed prompt-specific baselines computed from eight separate precomputed rollouts, controlling for baseline quality. We run this comparison using OLMo-3-7B-Think-DPO on math-only and multi-domain subsets of Dolci-Think-RL-7B dataset. As shown in Figure~\ref{fig:adam_grad_variance}, $m=8$ consistently yields higher gradient variance, with a larger gap in the multi-domain setting. Further details are provided in Appendix~\ref{app:gradient_variance}.

In other words, given the same total budget, using as many diverse prompts as possible is critical to stable learning. Yet GRPO requires repeated sampling from the same prompt to estimate a faithful baseline $b(x)$, This motivates our method of estimating a reliable baseline with minimal sampling without training a separate value network as in PPO. 

We provide an extended discussion of related work on value estimation in RL, outcome-relevant internal representations and multi-domain RL training in \S~\ref{related_work}.

\begin{figure}[t]
    \centering
    \begin{minipage}[t]{0.43\linewidth}
        \centering
        \includegraphics[width=\linewidth]{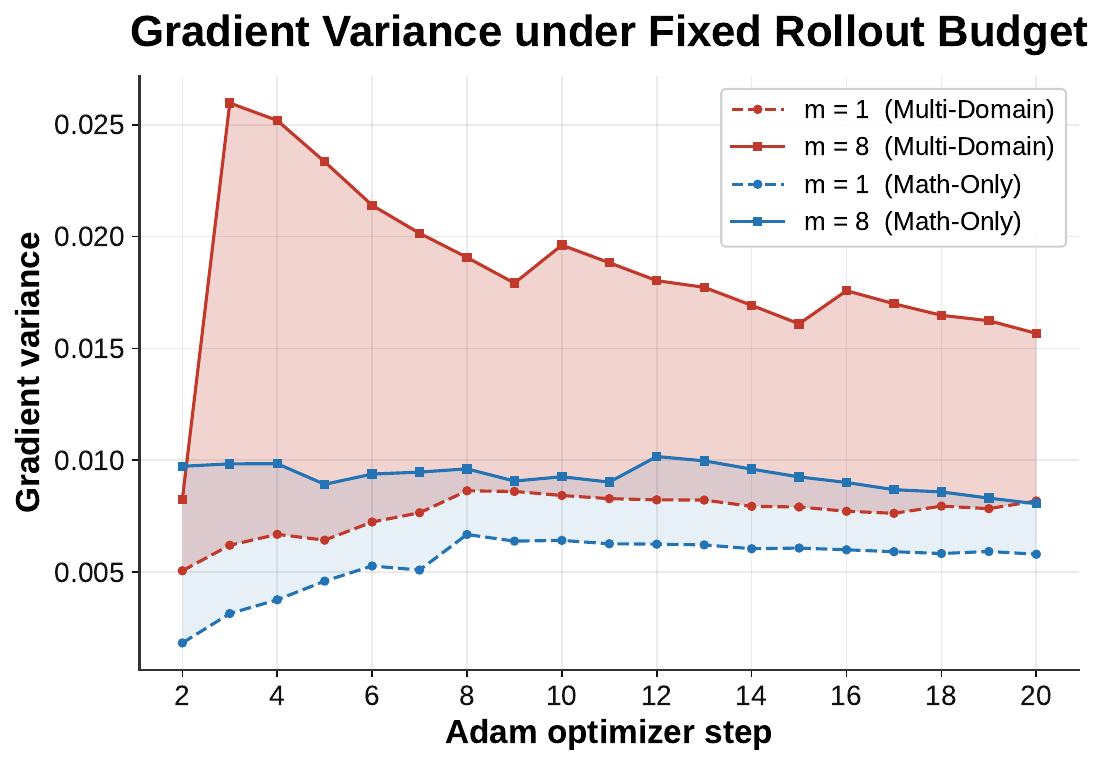}
        \caption{
        Estimated policy-gradient variance under a fixed total rollout budget.
        The $m=1$ condition allocates one rollout to each of $B$ distinct prompts,
        whereas the $m=8$ condition allocates eight rollouts to each of $B/8$
        prompts.
        }
        \label{fig:adam_grad_variance}
    \end{minipage}
    \hfill
    \hfill
    \begin{minipage}[t]{0.55\linewidth}
        \centering
        \includegraphics[width=\linewidth]{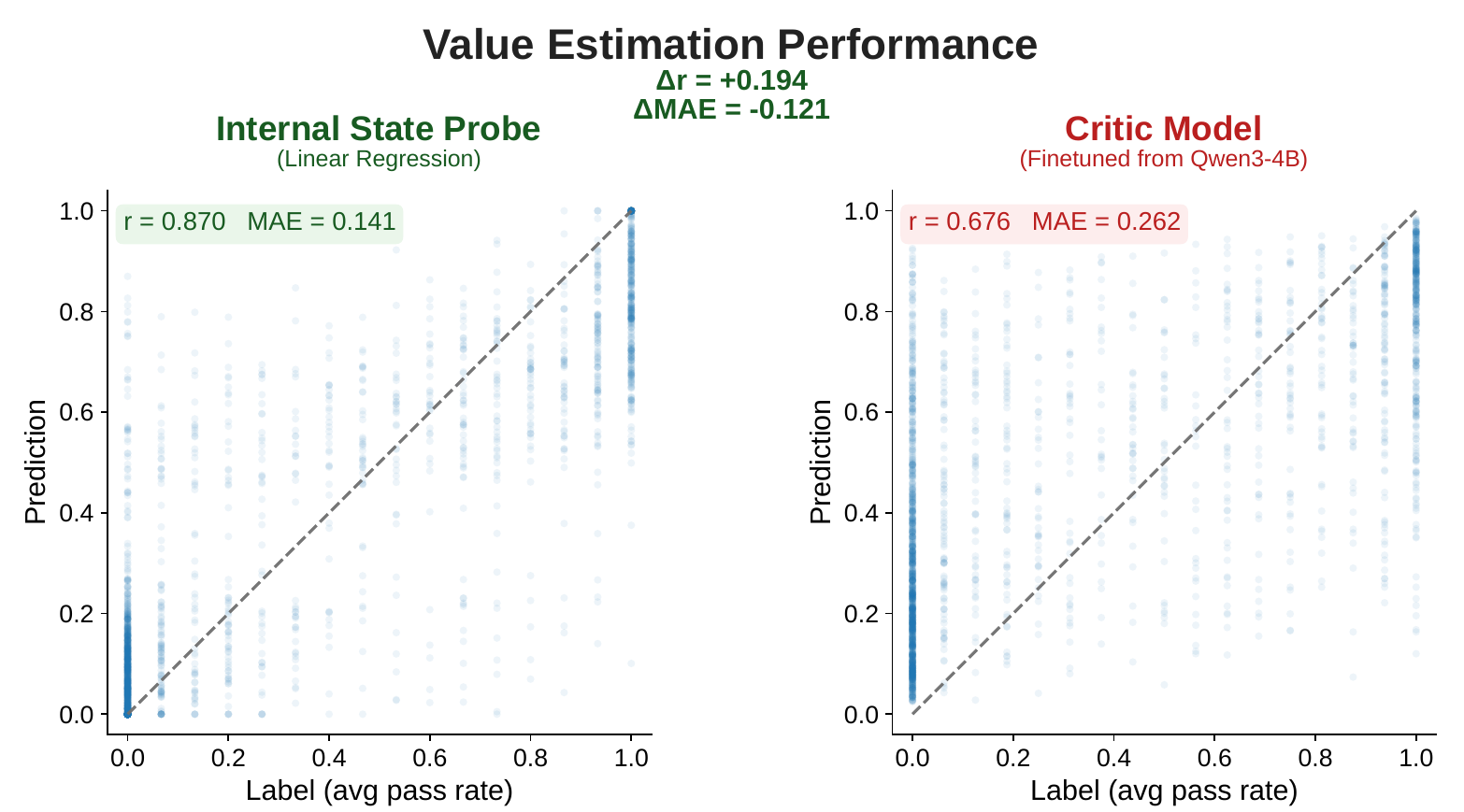}
        \caption{
        Value prediction between our probe and a separately
        trained critic. Predictions are compared against empirical Avg@8 scores.
        Our probe achieves higher Pearson correlation ($r$) and lower MAE, indicating
        that the policy's own internal representations provide an effective low-cost
        signal for value estimation.
        }
        \label{fig:value_estimator_prelim}
    \end{minipage}
\end{figure}


\section{Policy Optimization with Internal State Value Estimation (POISE)}
\label{sec:POISE}
We now introduce POISE, which leverages the policy model's internal state signals for value estimation in RLVR. We first show that a lightweight probe can predict the value function, i.e., the expected verifier reward, directly from the policy model's internal states (\S~\ref{sec:Estimating_Baseline_Values}). We then integrate this probe into policy optimization to compute per-rollout advantages, yielding the full POISE algorithm without requiring a separate LLM-scale value model (\S.~\ref{sec:cross_rollout}).

\subsection{Value Function Estimation from Policy Model Internal States}
\label{sec:Estimating_Baseline_Values}

We introduce a probe designed to estimate baseline values directly from the policy model's internal representations. Since the viability of this method hinges on the presence of such information, we additionally present preliminary empirical results demonstrating that these internal states inherently encode the necessary signals for accurate value estimation.

\paragraph{Probe prediction objective.}
The probe is trained to predict the prompt-level value under the current policy, defined as the expected verifier reward:
\[
    V^{\pi_\theta}(x)
    =
    \mathbb{E}_{y\sim\pi_\theta(\cdot|x)}
    [R(x,y)] .
\]
Since the ground-truth quantity is unknown, we instead sample \(K\) rollouts for each prompt \(x\),
\(y^{(1)},\ldots,y^{(K)} \sim \pi_\theta(\cdot|x)\), and collect their verifier rewards
\(r^{(i)} = R(x,y^{(i)}) \in \{0,1\}\). For the supervised example associated with rollout
\(y^{(i)}\), we use the leave-one-out Monte Carlo target as its gold value:
\[
    \widehat V_{-i}(x)
    =
    \frac{1}{K-1}
    \sum_{j\neq i} r^{(j)} .
\]
By excluding \(r^{(i)}\), \(\widehat V_{-i}(x)\) remains conditionally
independent of the input rollout \(y^{(i)}\) given \(x\), while still estimating
\(V^{\pi_\theta}(x)\) in expectation. This prevents the target from leaking the
reward of the same rollout whose features are used by the probe.

\paragraph{Probe input features.}
As shown in Figure~\ref{fig:methodology_figure} (left), each supervised example
for our probe is indexed by a prompt and one rollout, \((x,y^{(i)})\). For each
pair, we construct the probe input from three complementary signals produced
during the forward pass of the current policy \(\pi_\theta\). (All hidden-state
features are extracted from a fixed layer \(\ell\), which we omit below for
readability.) Let \(H_{\theta,t}^{(i)}\) denote the residual-stream hidden state
at token position \(t\) for \((x,y^{(i)})\), and let \(P\) and \(R^{(i)}\)
denote the final \(n\) prompt-token and reasoning-token positions, respectively.

First, we use the prompt-state feature
\(h_{\theta,p}^{(i)}=\mathrm{Avg}_{t\in P}H_{\theta,t}^{(i)}\), motivated by
evidence that prompt hidden states encode pre-generation estimates of difficulty
and capability boundaries~\citep{zhu-etal-2025-llm}. Second, we use the
reasoning-state feature
\(h_{\theta,r}^{(i)}=\mathrm{Avg}_{t\in R^{(i)}}H_{\theta,t}^{(i)}\), since
trajectory-level hidden states can expose value-relevant information not
available from the prompt states alone~\citep{zhang2025reasoning}. Third, we use
token-level entropy statistics \(u_\theta^{(i)}\) as lightweight uncertainty
features~\citep{wang2026beyond}: the mean predictive entropy over
the vocabulary, averaged separately over all generated tokens, reasoning tokens,
and answer tokens.
The final probe input is
\[
    \phi_\theta^{(i)}
    =
    [h_{\theta,p}^{(i)};\, h_{\theta,r}^{(i)};\, u_\theta^{(i)}].
\]
We ablate these input components in \S~\ref{sec:estimator-analysis} and the
hidden-state extraction hyperparameters in
\S~\ref{app:ablations_hidden_state}.
It is important to clarify that, while the input features include the generated reasoning, the estimator learns to predict the prompt-based value, rather than verifying its own reasoning, because the prediction target during training is the expected reward derived from other responses.

\paragraph{Probe implementation.}

We train lightweight regressors \(g_f\) to minimize the following loss.
\[
    \mathcal{L}_{\mathrm{value}}(f)
    =
    \mathbb{E}_{x,i}
    \left[
    \left(
    g_f(\phi^{(i)}_\theta) - \widehat V_{-i}(x)
    \right)^2
    \right].
\]
Although our framework can theoretically support any regression architecture, we implement the probe using linear regression because its computational efficiency allows for fast, lightweight updates at each training step. This lightweight design also enables domain-specific baseline estimation
by training separate probes for different domain groups, without
additional policy-model computation. We provide an ablation of probe designs in \S~\ref{app:ablations_probe_design}, and provide detailed implementations and hyperparameters in \S~\ref{app:hidden-state-value-estimator}.

\paragraph{Preliminary experiment.}

Before using the probe for policy optimization, we first test whether the policy model's internal states contain enough information to reliably estimate the prompt-level value across domains and model
scales. We use reward-labeled rollouts from DAPO-Math~\citep{yu2025dapoopensourcellmreinforcement}, DeepScaleR~\citep{deepscaler2025}, AceCoder~\citep{acecoder},
ToolDial~\citep{shim2025tooldial}, and IF-RLVR~\citep{pyatkin2026generalizing}, spanning mathematical reasoning, coding, tool calling, and instruction following. 
We
evaluate four policy backbones---Qwen3-4B,
DeepSeek-R1-Distill-Qwen-1.5B, DeepSeek-R1-Distill-Qwen-7B, and
GPT-OSS-120B---and, for each model--dataset pair, compare our probe with a
separately trained policy-model-scale critic~\citep{wang2026spposequencelevelppolonghorizon}. Both estimators
are trained on the same rollout data and evaluated against empirical Avg@8
values on held-out prompts.

We find that our probe is competitive with and often more accurate than the
critic across models and domains, despite requiring only a lightweight readout
over signals already produced by the actor. Figure~\ref{fig:value_estimator_prelim}
shows a representative result on Qwen3-4B with DAPO-Math, with full results in
Appendix~\ref{app:generalizability}. Unlike the sequence-level critic used in this comparison, the probe also observes trajectory-level states and token-level
uncertainty, giving it access to richer signals about the current policy. We additionally compare POISE with PPO using a token-level critic conditioned on generated prefixes and GAE in end-to-end policy optimization (Appendix~\ref{app:ppo_multidomain}).
However, these response-dependent features cannot be used naively as a baseline without biasing the policy gradient. We next introduce a cross-rollout
construction that retains their predictive benefit while preserving
unbiasedness. 

\begin{figure}
    \centerline{
    \includegraphics[
        width=0.9\columnwidth
    ]{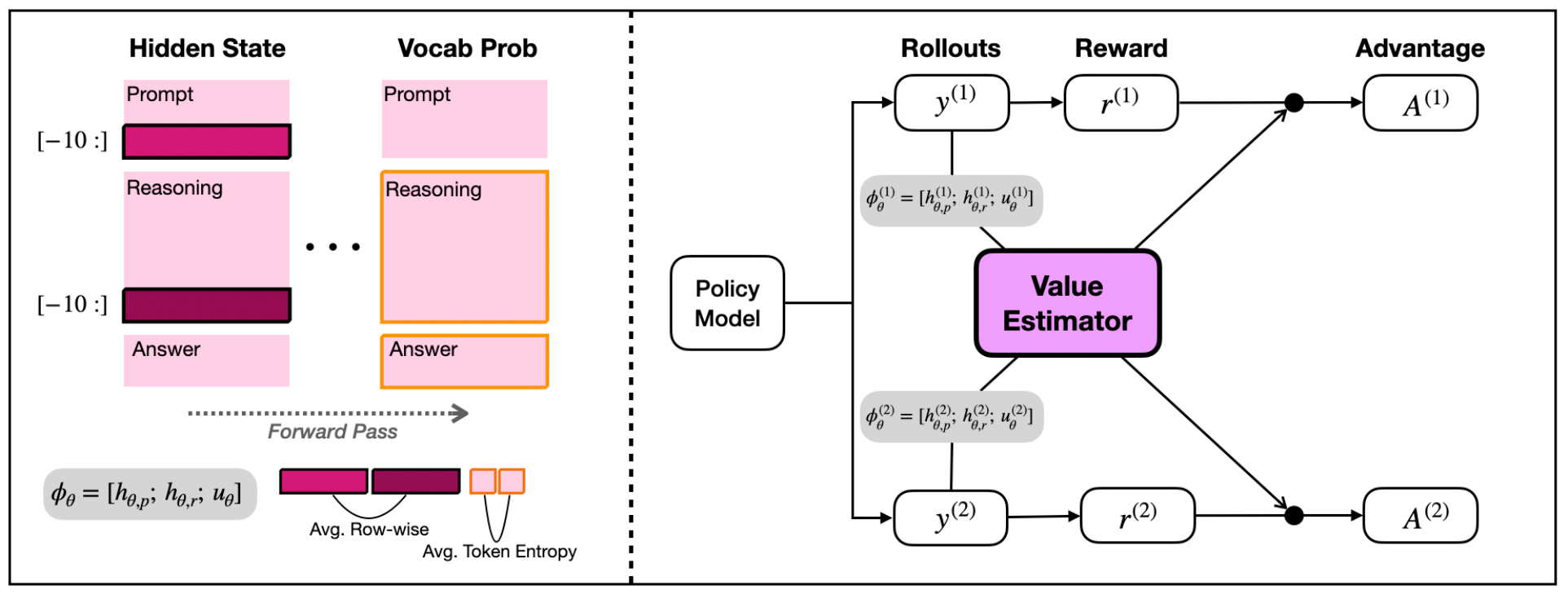}
    }
    \caption{Overview of POISE. \textbf{Left}: Probe features $\phi(x, y, \pi)$ combine hidden states with token entropy. \textbf{Right}: The value estimator predicts each rollout's baseline from the other rollout's features.}
    \label{fig:methodology_figure}
\end{figure}
\subsection{Policy Optimization with Cross-Rollout Baselines}
\label{sec:cross_rollout}

We now integrate the internal state probe into RL training as a value estimator, forming the full POISE algorithm (Figure~\ref{fig:methodology_figure}~right).

\paragraph{Two rollouts per prompt.}
For each prompt $x \sim \mathcal{D}$ in the training batch we sample two independent rollouts from the current policy,
\begin{equation}
    y^{(1)},\, y^{(2)} \;\overset{\mathrm{i.i.d.}}{\sim}\; \pi_{\theta_{\mathrm{old}}}(\cdot \mid x),
    \label{eq:two_rollouts}
\end{equation}
and evaluate their verifiable rewards $R(x, y^{(1)})$ and $R(x, y^{(2)})$.

\paragraph{Cross-rollout baseline and advantage.}
The baseline for each rollout is predicted from the internal signals of 
the \emph{other} rollout:
\begin{equation}
    b^{(1)}(x) \;=\; g_f\bigl(\phi^{(2)}),
    \qquad
    b^{(2)}(x) \;=\; g_f\bigl(\phi^{(1)}),
      \label{eq:cross_baseline}
\end{equation}
This yields the cross-rollout advantages
\begin{equation}
    A^{(i)}(x) \;=\; R(x, y^{(i)}) - b^{(i)}(x),
    \qquad i \in \{1, 2\}.
     \label{eq:cross_advantage}
\end{equation}
By construction, the baseline used to update $y^{(i)}$ depends only on the 
independently sampled rollout $y^{(j)}$, $j \neq i$, satisfying the 
conditional-independence condition in Eq.~\ref{eq:unbiased_baseline}.

\paragraph{PPO-style policy update.}
We optimize the policy with a PPO-style clipped surrogate objective.
Let $r_t^{(i)}(\theta) = \pi_\theta(y_t^{(i)} \mid x, y_{<t}^{(i)}) / \pi_{\theta_{\mathrm{old}}}(y_t^{(i)} \mid x, y_{<t}^{(i)})$ be the importance ratio at token $t$ of rollout $i$. The objective is
\begin{equation}
\label{eq:ppo_loss}
\begin{aligned}
\mathcal{L}(\theta)
=
\mathbb{E}_{x \sim D,\; y^{(1)}, y^{(2)} \sim \pi_{\theta_{\mathrm{old}}}(\cdot \mid x)}
\Bigg[
\frac{1}{2}
\sum_{i=1}^{2}
\frac{1}{|y^{(i)}|}
\sum_{t=1}^{|y^{(i)}|}
\min \Bigl\{
& r_t^{(i)}(\theta) A^{(i)}(x), \\
& \mathrm{clip}\!\left(
r_t^{(i)}(\theta),
1-\epsilon,
1+\epsilon
\right)
A^{(i)}(x)
\Bigr\}
\Bigg].
\end{aligned}
\end{equation}
which we maximize with respect to $\theta$.

\paragraph{Online estimator training with a trajectory buffer.}
The value estimator $g_f$ is trained jointly with the policy on a sliding buffer
of recent trajectories. At each step, for each prompt $x$ with two independent
rollouts $(y^{(1)}, y^{(2)})$, we construct value-estimator examples
$\{(x, \phi(x,y^{(i)}), \widehat V_{-i}(x))\}_{i=1,2}$, where
$\widehat V_{-i}(x)=R(x,y^{(j)})$, $ j \neq i$. To prevent leakage through the
buffer, the current baselines are computed with $g_f$ before it is updated on
the current batch. We then update $f$ by minimizing a regression loss over the
new examples and a buffer from recent steps. The buffer stabilizes the training
signal under policy drift, while the joint update keeps $g_f$ aligned with the
evolving policy. Because $g_f$ is a lightweight probe over signals already
computed during the forward pass, this update incurs negligible cost. The full
procedure is summarized in Algorithm~\ref{alg:vepo}.

\section{Experiments}
\label{sec:experiments}
\subsection{Experimental Setup}
\paragraph{Training.} We instantiate our method on Qwen3-4B~\citep{qwen3} and
OLMo-3-7B-Instruct-DPO~\citep{olmo2026olmo3}, and train on GURU~\citep{cheng2025guru},
a curated corpus of 92K verifiable prompts spanning six reasoning domains: math, code,
science, logic, simulation, and tabular reasoning. Each domain is paired with its own
programmatic verifier assigning a binary correctness reward. Rollouts are sampled with
temperature 1.0 and top-$p$ 1.0, and all methods are given an equal compute budget. 
For more details, refer to (\S~\ref{app:hidden-state-value-estimator}).

We compare against two critic-free baselines, which obtain their baseline from the group of rollouts. Throughout, $m$ denotes the number of rollouts sampled
per prompt. RLOO~\citep{ahmadian2024back} with $m = 2$ uses the
same rollout allocation as POISE but replaces our predicted value with the paired rollout's
reward, isolating the effect of value estimation. DAPO~\citep{yu2025dapoopensourcellmreinforcement}, a state-of-the-art GRPO-based RLVR
algorithm, serves as our main competitive baseline; we run it with $m = 8$ and adopt the
implementation of \citet{zheng2025actpaysefficientreinforcement}, which performs DAPO's
dynamic sampling more efficiently. Training details and full hyperparameters are provided
in \S~\ref{appendix:hyperparametersfortraining}.

\paragraph{Evaluation.} We follow the evaluation suite of GURU~\citep{cheng2025guru},
covering all six training domains as well as held-out task types. For math we use
MATH500~\citep{hendrycks2021math} and AIME 2025~\citep{maa_aime}; for code,
HumanEval~\citep{chen2021humaneval}, MBPP~\citep{austin2021mbpp}, and
LiveCodeBench~\citep{jain2025livecodebench}; for logic, Zebra Puzzle~\citep{lin2025zebralogic}
and ARC-AGI-1~\citep{chollet2019arc}; for simulation, CodeI/O~\citep{li2025codeio} and
CRUXEval~\citep{gu2024cruxeval}; for science, GPQA Diamond~\citep{rein2024gpqa} and
SuperGPQA~\citep{du2025supergpqa}; and for tabular reasoning, HiTab~\citep{cheng2022hitab},
MultiHiertt~\citep{zhao2022multihiertt}, and FinQA~\citep{chen2021finqa}. We additionally
report LiveBench~\citep{white2025livebench} and IFeval~\citep{zhou2023ifeval} to assess for task types not represented in training.

For each method we select a checkpoint using the validation split released with the training
corpus. 
All evaluations decode with temperature 0.6 and top-$p$ 0.95. The full evaluation protocol are provided in \S~\ref{app:eval-protocol}.

\subsection{Main Results}
\label{sec:main_results}
Table~\ref{tab:benchmark_results} reports benchmark-level performance. POISE achieves the
best overall average on both backbones, improving over DAPO by 1.35 points on Qwen3-4B
(58.10 vs.\ 56.75) and 1.96 points on OLMo-3-7B-Instruct-DPO (49.92 vs.\ 47.96). The
comparison with RLOO, which uses the same allocation of $m = 2$ but replaces our predicted
baseline with the paired rollout's reward, isolates the contribution of value estimation:
POISE leads it by 1.68 and 1.60 points, a margin comparable to its lead over DAPO. Notably,
RLOO with $m = 2$ is itself competitive with DAPO at $m = 8$, even surpassing it on
OLMo-3-7B-Instruct-DPO. This is consistent with the allocation effect in
\S\ref{sec:gradientvariance}: under a fixed rollout budget, prompt diversity can offset a
substantial loss in baseline quality, especially in the multi-domain setting. POISE keeps
that diversity while restoring a reliable baseline. The improvements are spread across the
suite rather than driven by a single benchmark: POISE scores highest among the trained
methods on 11 of 17 benchmarks for both backbones, and attains the best domain average on
four of the six training domains for Qwen3-4B and on all six for OLMo-3-7B-Instruct-DPO.
\begin{table}[t]
\centering
\small
\setlength{\tabcolsep}{3pt}
\renewcommand{\arraystretch}{1.08}
\caption{
Benchmark-level performance on Qwen3-4B and
OLMo-3-7B-Instruct-DPO. The best and second-best scores within each model are bolded and underlined, respectively.
}
\label{tab:benchmark_results}
\resizebox{\textwidth}{!}{%
\begin{tabular}{@{}l l rrrr @{\hspace{12pt}} rrrr@{}}
\toprule
&
&
\multicolumn{4}{c}{\textbf{Qwen3-4B}} &
\multicolumn{4}{c}{\textbf{OLMo-3-7B-Instruct-DPO}} \\
\cmidrule(lr){3-6}
\cmidrule(lr){7-10}
Domain
& Benchmark
& \shortstack{Base\\{\scriptsize }}
& \shortstack{RLOO\\{\scriptsize $m=2$}}
& \shortstack{DAPO\\{\scriptsize $m=8$}}
& \shortstack{POISE\\{\scriptsize $m=2$}}
& \shortstack{Base\\{\scriptsize }}
& \shortstack{RLOO\\{\scriptsize $m=2$}}
& \shortstack{DAPO\\{\scriptsize $m=8$}}
& \shortstack{POISE\\{\scriptsize $m=2$}} \\
\midrule
\multirow[t]{3}{*}{\textbf{Code}}
& HumanEval
& 87.20 & 93.90 & \textbf{95.73} & \underline{94.51}
& 15.24 & 87.80 & \underline{89.63} & \textbf{91.46} \\
& LiveCodeBench
& 22.58 & \underline{30.82} & 24.73 & \textbf{35.13}
& 6.81 & \textbf{32.62} & 25.81 & \underline{30.82} \\
& MBPP
& 80.20 & 83.40 & \underline{85.00} & \textbf{85.60}
& 53.00 & 67.20 & \underline{68.80} & \textbf{69.60} \\
\addlinespace[3pt]
\multirow[t]{2}{*}{\textbf{Logic}}
& ARC-AGI-1
& \underline{2.25} & \textbf{3.50} & \textbf{3.50} & \textbf{3.50}
& 1.00 & \textbf{2.25} & \underline{1.75} & \textbf{2.25} \\
& Zebra Puzzle
& 12.02 & \underline{42.39} & \textbf{43.77} & 41.05
& 18.06 & 24.37 & \underline{28.53} & \textbf{31.90} \\
\addlinespace[3pt]
\multirow[t]{2}{*}{\textbf{Math}}
& AIME(avg@32)
& 18.30 & 38.44 & \textbf{44.90} & \underline{43.44}
& 18.33 & 33.54 & \underline{33.85} & \textbf{35.83} \\
& MATH500
& 82.40 & 90.80 & \textbf{92.80} & \underline{92.00}
& 78.20 & 88.80 & \underline{91.40} & \textbf{92.20} \\
\addlinespace[3pt]
\multirow[t]{3}{*}{\textbf{Simulation}}
& CodeIO
& 11.00 & 13.50 & \underline{15.50} & \textbf{17.50}
& 10.50 & \underline{19.00} & 16.00 & \textbf{20.00} \\
& CRUXEval-I
& 87.00 & 86.12 & \underline{88.00} & \textbf{88.75}
& \textbf{80.12} & \underline{75.75} & 70.38 & 75.00 \\
& CRUXEval-O
& \underline{87.13} & 87.12 & 80.88 & \textbf{87.62}
& 49.12 & \underline{79.75} & 79.62 & \textbf{80.88} \\
\addlinespace[3pt]
\multirow[t]{2}{*}{\textbf{Science}}
& GPQA Diamond
& 42.93 & 54.04 & \underline{54.55} & \textbf{56.06}
& 40.40 & \underline{44.44} & 42.42 & \textbf{47.98} \\
& SuperGPQA
& 36.50 & 41.10 & \underline{43.40} & \textbf{44.30}
& 26.30 & \underline{30.80} & 30.30 & \textbf{31.40} \\
\addlinespace[3pt]
\multirow[t]{3}{*}{\makecell[tl]{\textbf{Tabular}\\\textbf{Reasoning}}}
& FinQA
& \textbf{46.05} & 44.78 & 44.50 & \underline{45.78}
& 28.61 & \underline{34.97} & 34.79 & \textbf{38.96} \\
& HiTab
& 65.00 & \underline{72.10} & 71.50 & \textbf{73.40}
& 36.10 & 56.50 & \underline{57.60} & \textbf{58.00} \\
& MultiHiertt
& 49.40 & 50.60 & \underline{51.49} & \textbf{52.08}
& 20.83 & \textbf{32.44} & 28.87 & \underline{31.55} \\
\addlinespace[3pt]
\textbf{LiveBench}
& LiveBench
& 39.59 & \textbf{45.78} & 45.20 & \underline{45.35}
& 30.18 & \underline{34.95} & \textbf{36.59} & 33.54 \\
\addlinespace[3pt]
\makecell[tl]{\textbf{Instruction}\\\textbf{Following}}
& IFEval
& \textbf{83.18} & 80.78 & 79.30 & \underline{81.70}
& \underline{78.56} & 76.34 & \textbf{78.93} & 77.26 \\
\midrule
\multicolumn{2}{l}{\textbf{Average}}
& 50.16 & 56.42 & \underline{56.75} & \textbf{58.10}
& 34.79 & \underline{48.32} & 47.96 & \textbf{49.92} \\
\bottomrule
\end{tabular}%
}
\end{table}

\paragraph{Additional experiments on instruction following.}
Our training corpus contains no training data for instruction following, so none of the methods is trained on this
domain in our main experiments. For completeness, we additionally train each method on the
IF-RLVR training split released with IFBench~\citep{pyatkin2026generalizing} and evaluate on the
IFBench test set and IFEval~\citep{zhou2023ifeval}. POISE achieves the best performance on
both benchmarks across the two backbones, indicating that its advantage carries over to
instruction following when in-domain data are available. Full results are provided in
Appendix~\ref{app:instruction_following}.
\begin{figure}[t]
  \centering
  \begin{minipage}[t]{0.50\textwidth}
    \centering
    \includegraphics[width=0.9\linewidth]{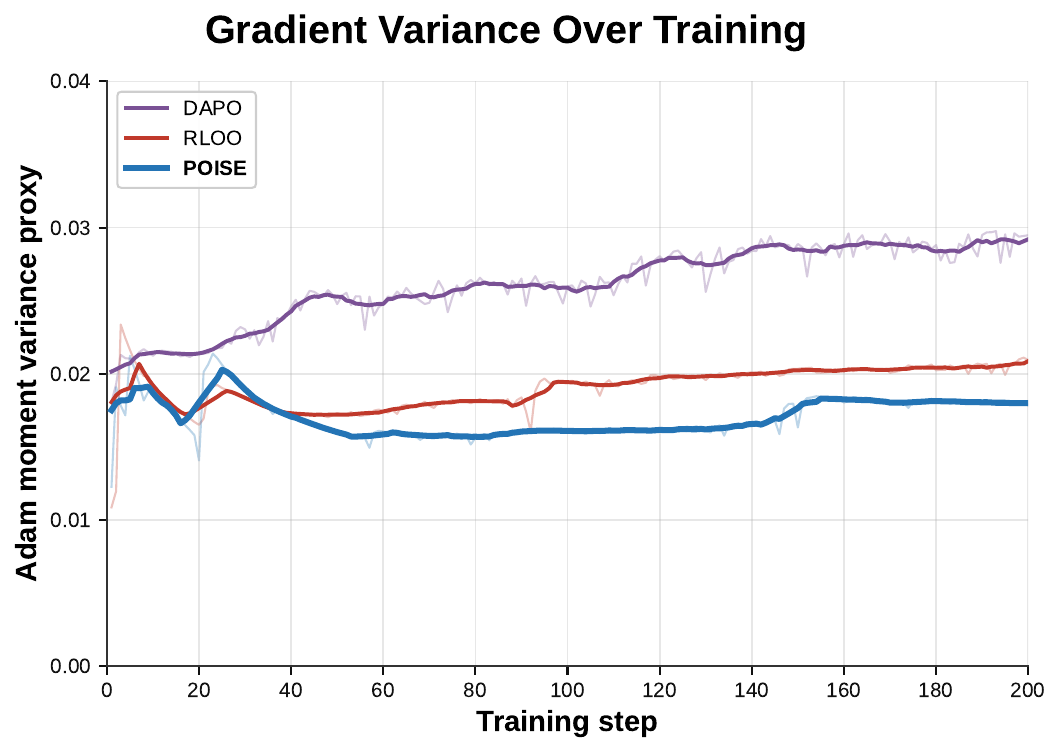}
    \captionof{figure}{Adam-moment variance proxy during training (Qwen3-4B).
The proxy is estimated from the optimizer's bias-corrected
    moments as $\mathrm{Var}(g) \approx \hat{v}_t - \hat{m}_t^2$ and summed across
    parameters.}
    \label{fig:gradvarpolicy}
  \end{minipage}
  \hfill
  \begin{minipage}[t]{0.48\textwidth}
    \centering
    \includegraphics[width=0.9\linewidth]{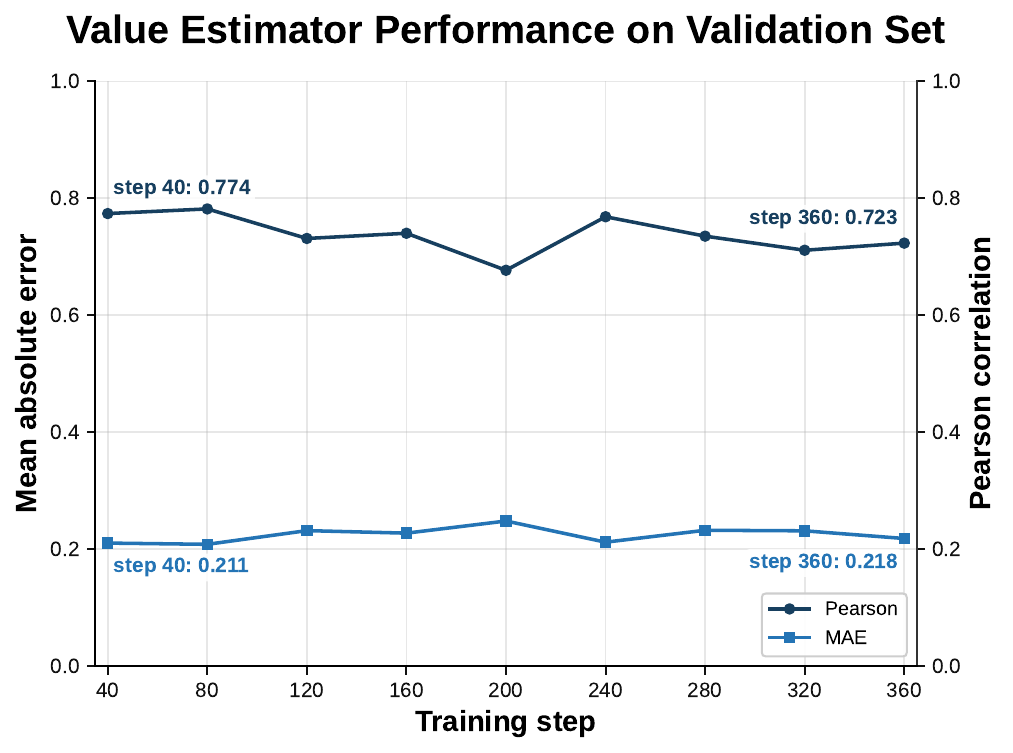}
    \captionof{figure}{Accuracy of the 
    value estimator over training steps on Qwen3-4B.
    Estimator predictions 
    are compared against an oracle 
    computed from 8 rollouts per prompt; we report MAE (left axis) and Pearson
    correlation (right axis).}
    \label{fig:valacc}
  \end{minipage}
\end{figure}

\paragraph{Training stability.}
We track a proxy for policy-gradient variance during training, estimated from the Adam optimizer state following~\citet{takashiro2026maxpo}. This is not the exact gradient variance, but it is practical at LLM scale and computed identically for all methods. Since DAPO divides its advantages by the group standard deviation by design, its values come out larger purely from this scale difference. For a fair comparison we multiply them back by the squared group standard deviation, estimated from the logged batch reward statistics.

Figure~\ref{fig:gradvarpolicy} shows the result. POISE maintains lower gradient variance than RLOO throughout training, indicating that predicting the baseline from internal states reduces variance beyond what the paired rollout's reward achieves at the same allocation. DAPO remains higher than both even after rescaling, and drifts upward as training proceeds. This suggests that insufficient prompt diversity may make training progressively less stable as the policy moves further from its initialization in a multi-domain setting.

\paragraph{Value estimator performance over training.}
We also track how the estimator itself holds up as the policy changes. Every 40 steps we
sample eight rollouts per prompt on the validation set to obtain an oracle baseline, and
compare the estimator's predictions against it. Figure~\ref{fig:valacc} shows that both MAE
and Pearson correlation remain stable throughout training, with no degradation as
the policy improves. Since the value function is nonstationary during RL, this stability
indicates that the online update keeps the estimator aligned with the evolving policy rather
than with its initialization.

\section{Analysis}
\label{sec:estimator-analysis}
We analyze the estimator through controlled ablations and compare POISE
with critic-based PPO in end-to-end RL. The ablations use Qwen3-4B in a controlled mathematical-reasoning setting based on DAPO-Math-17K training
set \citep{yu2025dapoopensourcellmreinforcement}. The RL comparison also includes the main
multi-domain experiments on both policy backbones.

\begin{table}[t]
\centering

\begin{minipage}[t]{0.422\textwidth}
\vspace{0pt}
\centering
\small
\captionof{table}{
Ablation of estimator input features on Qwen3-4B.
We report MAE and Pearson correlation ($r$) against empirical
Avg@8 values.
}
\label{tab:ablation-main}

\setlength{\tabcolsep}{4pt}
\renewcommand{\arraystretch}{1.08}

\resizebox{\linewidth}{!}{%
\begin{tabular}{lcc}
\toprule
Input feature & MAE $\downarrow$ & $r$ $\uparrow$ \\
\midrule
Prompt hidden states only
& 0.234 & 0.569 \\
Reasoning hidden states only
& 0.132 & 0.821 \\
Mean entropy only
& 0.152 & 0.780 \\
Response length only
& 0.251 & 0.494 \\
\midrule
Full estimator
& \textbf{0.126} & \textbf{0.838} \\
\bottomrule
\end{tabular}%
}
\end{minipage}
\hfill
\begin{minipage}[t]{0.55\textwidth}
\vspace{0pt}
\centering
\small
\captionof{table}{
End-to-end ablations and comparison with a policy-model-scale critic
on Qwen3-4B. We report Avg@32 over seven mathematical-reasoning benchmarks.
}
\label{tab:rl-design-ablation}
\label{tab:critic-rl-comparison}

\setlength{\tabcolsep}{6pt}
\renewcommand{\arraystretch}{1.05}

\resizebox{\linewidth}{!}{%
\begin{tabular}{lcc}
\toprule
Method & Rollouts & Avg@32 \\
\midrule
base model & -- & 25.8 \\
POISE (no trajectory-level feature)       & 1 & 47.5 \\
POISE (no cross-rollout)  & 1 & 40.7 \\
PPO (Token-level critic)  & 1 & 49.3 \\
\midrule
POISE (Ours)   & 2 & \textbf{50.1} \\
\bottomrule
\end{tabular}%
}
\end{minipage}
\end{table}

\paragraph{Ablations.}
We begin by evaluating the estimator itself, holding the policy and data fixed
while varying its input features, hidden-state extraction settings, and probe
architecture. As shown in Table~\ref{tab:ablation-main}, all of our three
internal signals (prompt hidden states, reasoning hidden states, and mean token
entropy) contain value-relevant information. However, combining all three features gives the
best overall performance, suggesting that the features capture complementary
aspects of prompt difficulty and policy uncertainty. Response length, by
contrast, is substantially less predictive, indicating that the estimator does
not simply rely on a superficial length signals.
We further vary the extraction layer and pooling window, and compare the linear
probe with several MLP architectures. The estimator performs well across a broad
range of layers and is relatively insensitive to the pooling window. The linear
probe also matches or surpasses the MLP probes in Pearson correlation, despite
being substantially cheaper. Detailed results are
reported in Appendices~\ref{app:ablations_hidden_state}
and~\ref{app:ablations_probe_design}.

We next test whether these design choices actually improve online policy optimization
(Table~\ref{tab:rl-design-ablation}). We first compare full POISE with a
prompt-only estimator using a single rollout. 
Although this comparison does not isolate the effect of trajectory-level features, the prompt-only variant benefits from twice as many distinct prompts under the same rollout budget. Full POISE nevertheless performs better, supporting the practical benefit of the two-rollout design despite its lower prompt diversity.
We then remove the cross-rollout construction while
retaining the full feature set. The resulting same-rollout variant performs
substantially worse, confirming that cross-rollout estimation is necessary when using trajectory-dependent features. 

\paragraph{End-to-end RL comparison to a policy-model-scale critic.}
We compare POISE with PPO equipped with a separate policy-model-scale
critic. The critic predicts token-level values, is trained jointly with
the policy, and forms advantages using generalized advantage estimation.
In the controlled mathematical-reasoning setting,
Table~\ref{tab:critic-rl-comparison} shows that POISE achieves a higher
downstream score than critic-based PPO.

We extend this comparison to the main multi-domain setting on Qwen3-4B
and OLMo-3-7B-Instruct-DPO. POISE achieves a higher overall benchmark
average on both models (Table~\ref{tab:ppo_domain_comparison}). This result shows that the actor's own representations can provide an effective
value signal without training an additional policy-scale model. Because critic cost grows with the actor while POISE adds only a lightweight readout, this advantage becomes particularly relevant at larger scales. The full PPO setup and domain-level results are reported in Appendix~\ref{app:ppo_multidomain}.
We further show that value remains decodable across diverse domains
and models, including GPT-OSS-120B, in
Appendix~\ref{app:generalizability}.

\section{Conclusion}
We introduce \textbf{POISE} (Policy Optimization with Internal State Value Estimation), which
predicts rollout value from the policy's internal states instead of relying on a group-mean baseline or a separate critic. To preserve unbiasedness, we couple this baseline with a cross-rollout construction. In multi-domain RL training over a six-domain verifiable-reward corpus, our method consistently outperforms existing RLVR baselines while keeping optimization stable. We further show that the value estimator closely tracks the evolving policy throughout training, 
generalizes across diverse task types, and scales robustly to substantially larger models. 

\section{Limitations and Future Work}
\label{sec:limitations}
Our experiments are conducted under a fixed compute budget, and while the trends we report are consistent across backbones and benchmarks, characterizing the behavior of internal state value estimation under substantially longer training horizons remains an interesting direction that we leave to future work with greater compute resources.

Several extensions naturally follow from our framework. First, our current estimator predicts value at the sequence level; extending it to token-level credit assignment would yield finer-grained advantages that more precisely reward the tokens responsible for a successful trajectory and penalize those that derail it, an effect known to be especially impactful for long reasoning trajectories~\citep{lin2025criticaltokensmattertokenlevel, wang2026beyond}. Second, internal state value estimates may also be useful beyond policy gradient optimization: for preference learning algorithms such as Direct Preference Optimization~\citep{rafailov2024directpreferenceoptimizationlanguage}, they could inform the construction of response pairs by identifying rollouts with meaningfully different predicted values, potentially yielding more informative preference comparisons. Third, we trained our estimator on verifiable rewards, but many non-verifiable tasks are instead scored by rubrics or generative judges~\citep{gunjal2025rubrics, su2025crossing}. Whether a baseline for such graded and noisier rewards can likewise be predicted from internal states remains open. Finally, long-horizon and agentic tasks make group sampling even harder to afford, since a single rollout spans many turns~\citep{hou2026sao}. Reading value from the internal states carried across those turns, and tracking how that value shifts from one turn to the next, is a promising direction for obtaining a baseline in this setting.

\medskip

\bibliographystyle{plainnat}
\bibliography{iclr2027_conference}
\appendix

\clearpage

\addtocontents{toc}{\protect\setcounter{tocdepth}{2}}

\begingroup
\renewcommand{\contentsname}{Appendix Contents}
\tableofcontents
\endgroup

\clearpage

\clearpage
\section{Related Work}
\label{related_work}
\noindent \textbf{Value Estimation in RL for LLM Reasoning. }
RL algorithms for LLMs reasoning differ mainly in how they estimate the
baseline that reduces policy-gradient variance~\citep{schulman2017proximal,
ouyang2022training, shao2024deepseekmath}. Recent work extends this along
three axes. The first introduces explicit value models either a generalist
value prior for sparse rollouts~\citep{zhang2026v05generalistvaluemodel}
or a sequence-level value model that treats reasoning as a contextual
bandit~\citep{wang2026spposequencelevelppolonghorizon} but incurs the
substantial training, calibration, and deployment cost of an additional
LLM-scale model. \citet{chen2026hista} instead estimate the value of a state
non-parametrically, averaging the returns of the states closest to it in
hidden-state space, which requires a large group of rollouts per prompt
to populate the neighborhood. \citet{xu2025singlestreampolicyoptimization}
forgo groups, tracking a persistent per-prompt value estimate across
training steps. The tracker is tabular and thus carries no information
for a prompt it has not yet seen, so each prompt is initialized from
substantial offline rollouts; without them the estimate only becomes
informative after a prompt has been revisited over epochs, which is
infeasible on corpora large enough that most prompts are seen only once.
The second axis reduces rollout cost by skipping uninformative prompts
before generation, predicting their informativeness from past reward
dynamics~\citep{zheng2025actpaysefficientreinforcement}, which likewise
depends on prompts being revisited across training. The third refines
how the empirical-mean baseline is formed: \citet{zeng2025shrinkage} apply
the James--Stein principle to interpolate between the per-prompt mean
and the batch mean, obtaining a lower-MSE baseline at no additional cost.

The first two axes obtain a better baseline by paying for it, either
with an auxiliary model or with additional rollouts. Our method instead
predicts the baseline at negligible cost from the policy's hidden states
and generation signals already computed during the forward pass,
requiring neither an auxiliary value model nor more than a pair of
rollouts per prompt, while preserving the variance-reduction benefit
of value estimation. The third axis asks how to best combine the reward
means already available in a batch, whereas ours asks how to obtain a
more accurate baseline in the first place; the two approaches are
therefore composable and using our predicted value as the shrinkage
target is a natural extension we leave to future work.

The idea of obtaining a baseline from existing model outputs also
appears in on-policy distillation (OPD)~\citep{lu2025onpolicydistillation}. \citet{oh2026kl} formulate OPD
as policy-gradient RL and derive its per-token value baseline directly
as the negative reverse KL between the student and teacher distributions.
This baseline is available from the existing forward pass without an
additional critic or extra rollouts, illustrating how the structure of
the distillation reward enables direct value computation.

\noindent \textbf{Outcome-relevant Information in Hidden States.}
A growing body of work shows that the hidden states of large language models encode information relevant to assessing their outputs and task outcomes. Early studies on language model interpretability support this by showing that simple linear probes over internal representations can recover latent properties such as factuality~\citep{burns2023discovering}, truthfulness~\citep{azaria-mitchell-2023-internal, ITI_paper}, confidence~\citep{zou2023representation}, and answer correctness{}. Recent studies extend this idea to reasoning models, where activations have been used to predict final correctness~\citep{cencerrado2026no}, identify capability boundaries between solvable and unsolvable prompts~\citep{zhu-etal-2025-llm}, estimate perceived difficulty and reasoning effort~\citep{zhang2025reasoning}, and support self-verification or early stopping during generation~\citep{sui2025stop}. Moving beyond black-box probes over hidden states, \citet{xu2026sparserewardsubsystemlarge} localize this information to a sparse subset of neurons, identifying value neurons whose activations predict state value and dopamine neurons that encode step-level TD errors, yet these too are deployed only at inference time.

Overall, prior work primarily uses hidden-state information as diagnostic signals or test-time control mechanisms. In contrast, we incorporate such signals directly into RL training by learning an online value estimator from the policy model's own hidden states, yielding a cheap baseline without requiring an auxiliary LLM-scale critic or large rollout groups.

\noindent \textbf{Multi-Domain RL.}
Modern post-training pipelines combine verifiable tasks from many domains, and differ
mainly in how those domains are combined. One line keeps domains apart during
optimization: either training them in separate sequential stages, or training per-domain
experts and consolidating them afterwards through weight merging or on-policy
distillation \citep{yang2026cascade, ma2026mopd, wang2026mix}. The other mixes all domains
into a single RL run. Recent evidence favors the latter. \citet{wang2026mix} find that mixed multi-task RLVR matches separate-then-merge pipelines at substantially lower compute cost, and \citet{nvidia2025nemotron} report that training all tasks simultaneously is
more stable and less prone to reward hacking than the staged pipelines they previously
used, which often degraded capabilities acquired in earlier stages. Following this
practice, we train in the mixed setting, where each batch contains prompts drawn from all
domains \citep{cheng2025guru}.

Given heterogeneous training distributions, a growing body of work asks how to train on
them well. These efforts largely fall into three groups: controlling what enters the
batch, through domain-mixture optimization, transfer-aware curricula, or samplers that
align each task's realized contribution with its intended weight \citep{liang2025modomodo,
yang2026tac, ramesh2026mtgrpo}; reconciling conflicting updates once they arise, through
gradient surgery, task reweighting, or capability-subspace projection \citep{cai2026mgs,
ramesh2026mtgrpo, zhang2026carerl}; and shaping cross-domain representations with
auxiliary objectives \citep{yu2026mcpo}. Underlying these is the observation that
heterogeneous tasks contribute gradients of widely differing magnitudes, so naive mixing
implicitly biases optimization toward large-gradient domains \citep{wu2025imbalanced}.

All of these operate either on what enters the batch or on the update it produces. Our work asks a different question: how the RL algorithm itself should be modified to train well on such batches. The starting point is that the number of distinct prompts per batch is what governs the stability of the gradient estimate, so an algorithm for this regime should be designed to preserve prompt diversity under a fixed rollout budget. POISE achieves this by estimating the baseline from the policy's own internal states, freeing the remaining budget for distinct prompts. Because it touches neither of the two levers above, it remains complementary to all of these approaches.

\clearpage

\section{Theoretical Proofs}
\label{app:theoretical_proofs}
\subsection{Proof of Proposition 1}
\label{proof:prop1}
\begin{proof}
Define:
\begin{equation*}
    \mu(x) = \mathbb{E}[Z(x,y) \mid x], \quad \Sigma_w = \mathbb{E}_x[\operatorname{Cov}(Z(x,y)\mid x)], \quad \Sigma_b = \operatorname{Cov}_x(\mu(x)).
\end{equation*}
Let $G_i = \frac{1}{m}\sum_{j=1}^m Z(x_i, G_{ij})$. Since completions within a prompt are conditionally independent given $x_i$,
\begin{equation*}
    \operatorname{Cov}(G_i \mid x_i) = \frac{1}{m}\operatorname{Cov}(Z(x_i, y)\mid x_i).
\end{equation*}
Applying the law of total covariance to $G_i$:
\begin{equation*}
    \operatorname{Cov}(G_i) = \mathbb{E}_x[\operatorname{Cov}(G_i \mid x_i)] + \operatorname{Cov}_x(\mathbb{E}[G_i \mid x_i]) = \frac{1}{m}\Sigma_w + \Sigma_b.
\end{equation*}
Since prompts are sampled independently, $\operatorname{Cov}(\hat{g}) = \frac{1}{n}\operatorname{Cov}(G_i)$. Substituting $n = B/m$:
\begin{equation*}
    \operatorname{Cov}(\hat{g}) = \frac{m}{B}\!\left(\frac{1}{m}\Sigma_w + \Sigma_b\right) = \frac{1}{B}\Sigma_w + \frac{m}{B}\Sigma_b. \qedhere
\end{equation*}
\end{proof}

\subsection{Proof of Corollary 1}
\label{proof:cor1}
\begin{proof}
From Proposition~1, for any allocation $(n, m)$ with $nm = B$,
\begin{equation*}
    \operatorname{Cov}(\hat{g}) = \frac{1}{B}\Sigma_w + \frac{m}{B}\Sigma_b.
\end{equation*}
Consider any two allocations $m_1 < m_2$ with the same budget $B$. Their difference is
\begin{equation*}
    \operatorname{Cov}(\hat{g}_{m_2}) - \operatorname{Cov}(\hat{g}_{m_1}) = \frac{m_2 - m_1}{B}\,\Sigma_b.
\end{equation*}
Since $\Sigma_b = \operatorname{Cov}_x(\mu(x))$ is a covariance matrix, it is positive semidefinite, so for any vector $v$,
\begin{equation*}
    v^\top \bigl(\operatorname{Cov}(\hat{g}_{m_2}) - \operatorname{Cov}(\hat{g}_{m_1})\bigr)v = \frac{m_2 - m_1}{B}\, v^\top \Sigma_b\, v \geq 0.
\end{equation*}
Therefore $\operatorname{Cov}(\hat{g}_{m_2}) \succeq \operatorname{Cov}(\hat{g}_{m_1})$ for any $m_2 > m_1$, meaning the variance of $\hat{g}$ in every gradient direction is non-decreasing in $m$. The minimum is thus attained at $m = 1$, giving $n = B$. \qedhere
\end{proof}
\clearpage

\section{Details of the Gradient-Variance Experiment}
\label{app:gradient_variance}

\begin{table}[h]
\centering
\begin{tabular}{p{0.42\linewidth}p{0.34\linewidth}}
\toprule
\textbf{Hyperparameter} & \textbf{Value} \\
\midrule
policy model              & Olmo-3-7B-Think-DPO \\
training data             & \texttt{allenai/Dolci-Think-RL-7B} \\
training steps            & 20 \\
train batch size          & 64 \\
mini batch size           & 64 \\
micro batch size per GPU  & 1 \\
max prompt length         & 2048 \\
max response length       & 32680 \\
learning rate             & $1 \times 10^{-6}$ \\
LR schedule               & constant (no warmup) \\
optimizer                 & AdamW  \\
weight decay              & 0.01 \\
gradient clipping         & 1.0 \\
clip ratio (low / high)   & 0.2 / 0.28 \\
use KL loss               & False \\
KL in reward              & False \\
sampling temperature      & 1.0 \\
sampling top-p            & 1.0 \\
max batched tokens        & 34728 \\
GPUs                      & 4 $\times$ 1 node \\
\bottomrule
\end{tabular}
\caption{Key training hyperparameters used in the experiment.}
\label{tab:hyperparameters_dolci_oracle}
\end{table}

\paragraph{Experimental objective.}
The purpose of this experiment is to isolate the effect of rollout allocation on policy-gradient variance. In particular, we test whether allocating a fixed rollout budget across fewer distinct prompts increases gradient variance, as predicted by Proposition~\ref{proposition_gradient_variance}. This experiment does not compare different baseline estimators or training algorithms; instead, it holds the baseline fixed and varies only the number of rollouts sampled per prompt.

\paragraph{Policy and data.}
We conduct the experiment using Olmo-3-7B-Think-DPO and prompts from
Dolci-Think-RL-7B. We consider two prompt distributions. In the math-only
setting, prompts are drawn exclusively from the mathematical reasoning subset.
In the multi-domain setting, prompts are sampled uniformly at random from the
full dataset without any domain filtering, so the domain composition follows
the natural mixture of Dolci-Think-RL-7B. In both settings the data seed is held fixed across rollout-allocation conditions, so that the two conditions see the same prompts in the same order and differ only in how the rollout budget is distributed.

\paragraph{Oracle baseline and advantage estimation.}
We use the pass rate that Dolci-Think-RL-7B already provides for each prompt and it is held fixed across all rollout-allocation conditions.

Under a fixed budget of trajectories per optimizer step, we compare two
allocations: $m = 1$ draws a single rollout from each of $B$ distinct prompts,
while $m = 8$ draws eight rollouts from each of $B/8$ prompts. For every sampled
trajectory $(x, y)$, the advantage is
\begin{equation}
A(x, y) = R(x, y) - b_{\mathrm{orc}}(x),
\label{eq:oracle_advantage_appendix}
\end{equation}
with no group centering and no standard-deviation normalization. Because
$b_{\mathrm{orc}}(x)$ does not depend on the sampled rollouts, the advantage
remains well defined at $m = 1$, and both allocations are trained against an
identical baseline. Any observed difference in gradient variance therefore
reflects the rollout allocation itself.

\paragraph{Gradient variance.}
Computing the gradient variance exactly is impractical at LLM scale, so we adopt
the Adam-state proxy of \citet{takashiro2026maxpo}: after bias correction, the
element-wise variance is estimated as $\widehat{\mathrm{Var}}(g) \approx \hat{v}_t
- \hat{m}_t^{2}$ and summed over all parameters. This is not the exact
instantaneous gradient variance, but it is computed identically in both
conditions and is therefore suitable for comparison. Because the two conditions
share the same policy checkpoint, total rollout budget, and fixed oracle
baseline, this measurement isolates the variance induced by the number of
distinct prompts in the batch. As shown in Figure~\ref{fig:adam_grad_variance}, the $m = 8$
allocation yields a higher variance proxy than $m = 1$ at every optimizer step in
both settings, and the separation is larger under the multi-domain distribution.
Prompt diversity therefore has a direct effect on gradient variance, and that effect is strongest when the training distribution is heterogeneous. Preserving diversity under a fixed budget requires obtaining a reliable baseline from as few rollouts as possible, which is the problem we address.
\clearpage

\section{Implementation Details}
\label{implementation}
\subsection{Pseudocode for POISE}
\begin{algorithm}[H]
\caption{POISE}
\label{alg:vepo}
\begin{algorithmic}[1]
\Require Prompt distribution $\mathcal{D}$; initial policy $\pi_{\theta_0}$;
initial value estimator $g_{f_0}$; prompt batch size $M$;
value-buffer capacity $N$; PPO clip $\epsilon$.
\State Initialize $\theta \gets \theta_0$, $f \gets f_0$, value buffer $\mathcal{B}_V \gets \emptyset$.
\For{step $= 1, 2, \ldots, T$}
    \State $\theta_{\mathrm{old}} \gets \theta$
    \State Sample a mini-batch of prompts $\{x_b\}_{b=1}^{M} \sim \mathcal{D}$.
    \State $\mathcal{R} \gets \emptyset$, $\mathcal{S}_V \gets \emptyset$
    \For{each prompt $x_b$}
        \State Sample
        $y_b^{(1)}, y_b^{(2)} \overset{\mathrm{i.i.d.}}{\sim}
        \pi_{\theta_{\mathrm{old}}}(\cdot \mid x_b)$.
        \State Extract internal state features
        $\phi_b^{(i)} \gets \phi_{\theta_{\mathrm{old}}}(x_b,y_b^{(i)})$
        for $i \in \{1,2\}$ via forward hooks.
        \State Compute rewards $r_b^{(i)} \gets R(x_b, y_b^{(i)})$ for $i \in \{1,2\}$.
        \State Compute cross-rollout baselines:
        \Statex \hspace{1.4em}
        $b_b^{(1)} \gets g_f\bigl(\phi_b^{(2)}\bigr), \quad
         b_b^{(2)} \gets g_f\bigl(\phi_b^{(1)}\bigr).$
        \State Compute advantages
        $A_b^{(i)} \gets r_b^{(i)} - b_b^{(i)}$ for $i \in \{1,2\}$.
        \State Add policy-update examples:
        \Statex \hspace{1.4em}
        $\mathcal{R} \gets \mathcal{R} \cup
        \bigl\{\bigl(x_b, y_b^{(i)}, A_b^{(i)}\bigr)\bigr\}_{i=1,2}$.
        \State Add value-estimator examples:
        \Statex \hspace{1.4em}
        $\mathcal{S}_V \gets \mathcal{S}_V \cup
        \left\{
        \bigl(\phi_b^{(1)}, r_b^{(2)}\bigr),
        \bigl(\phi_b^{(2)}, r_b^{(1)}\bigr)
        \right\}$.
    \EndFor
    \State Update $\theta$ by maximizing the PPO objective in Eq.~\ref{eq:ppo_loss} on $\mathcal{R}$.
    \State Update $f$ by minimizing
    \[
    \sum_{(\phi,\widetilde r)\in \mathcal{B}_V \cup \mathcal{S}_V}
    \bigl(g_f(\phi)-\widetilde r\bigr)^2 .
    \]
    \State Append $\mathcal{S}_V$ to $\mathcal{B}_V$; retain only the $N$ most recent examples.
\EndFor
\State \Return $\pi_\theta$, $g_f$.
\end{algorithmic}
\end{algorithm}
\subsection{Baseline Algorithm}
\label{app:dapo}

We adopt DAPO~\citep{yu2025dapoopensourcellmreinforcement} as our baseline RL algorithm.
For each prompt $q\sim\mathcal{D}$, a group of $G$ responses
$\{o^{(i)}\}_{i=1}^{G}$ is sampled from the old policy $\pi_{\theta_{\text{old}}}$,
and the following objective is optimized:
\begin{equation}
\begin{aligned}
\mathcal{J}_{\text{DAPO}}(\theta)
&= \mathbb{E}_{(q,a)\sim\mathcal{D},\,\{o^{(i)}\}_{i=1}^{G}\sim\pi_{\theta_{\text{old}}}(\cdot\mid q)} \\
&\quad \Bigg[\,\frac{1}{\sum_{i=1}^{G}|o^{(i)}|}
\sum_{i=1}^{G}\sum_{t=1}^{|o^{(i)}|}
\min\!\Big(\, r_{t}^{(i)}(\theta)\,\hat{A}_{t}^{(i)},\;
\mathrm{clip}\!\big(r_{t}^{(i)}(\theta),\, 1-\varepsilon_{\text{low}},\, 1+\varepsilon_{\text{high}}\big)\hat{A}_{t}^{(i)} \Big)\Bigg],
\end{aligned}
\label{eq:dapo}
\end{equation}
\begin{equation*}
\text{s.t.}\quad 0 \;<\; \big|\{\,o^{(i)} \mid \mathrm{is\_equivalent}(a, o^{(i)})\,\}\big| \;<\; G,
\end{equation*}
where the importance ratio and the group-relative advantage are
\begin{equation*}
r_{t}^{(i)}(\theta)=\frac{\pi_{\theta}(o_{t}^{(i)}\mid q,o_{<t}^{(i)})}{\pi_{\theta_{\text{old}}}(o_{t}^{(i)}\mid q,o_{<t}^{(i)})},
\qquad
\hat{A}_{t}^{(i)}=\frac{R^{(i)}-\mathrm{mean}(\{R^{(j)}\}_{j=1}^{G})}{\mathrm{std}(\{R^{(j)}\}_{j=1}^{G})}.
\end{equation*}

DAPO augments GRPO with four key techniques:
(i) \textbf{Clip-Higher} decouples the lower and upper clipping bounds
($\varepsilon_{\text{high}}>\varepsilon_{\text{low}}$), giving low-probability tokens
more room to be promoted and mitigating entropy collapse;
(ii) \textbf{Dynamic Sampling} filters out prompts whose responses are all correct
or all incorrect, ensuring that every batch yields effective gradients;
(iii) \textbf{Token-Level Policy Gradient Loss} normalizes by the total token count
$\sum_i|o^{(i)}|$ rather than per-sequence, so longer responses contribute proportionally
to the loss;
(iv) \textbf{Overlong Reward Shaping} applies a length-aware penalty to truncated
samples to reduce reward noise.

\paragraph{Reward.}
Unlike the original DAPO, we drop the length penalty and use a purely binary reward based solely on correctness for simplicity:
\begin{equation*}
R^{(i)} \;=\;
\begin{cases}
1, & \text{if } o^{(i)} \text{ is correct},\\
0, & \text{otherwise.}
\end{cases}
\end{equation*}

\subsection{Value Estimator}
\label{app:hidden-state-value-estimator}

Hidden states used for estimator training are collected from the model's
teacher-forced log-probability forward pass, which is already required by our
policy optimization algorithm and therefore incurs no additional policy-model
forward pass. Specifically, our implementation registers a forward hook on a
single transformer layer during this computation and separately pools the
prompt and reasoning hidden states.

For the single-domain mathematical-reasoning
(\S~\ref{app:controlled_rl_ablations}) and instruction-following
(\S~\ref{app:instruction_following}) experiments, we use a single estimator.
In the main multi-domain experiments, however, the domain label of each training
prompt is known, allowing us to use domain-specific estimators without evaluating
multiple probes for each trajectory. Each trajectory is routed to exactly one
lightweight estimator, so this design introduces no additional policy-model
computation relative to using a single estimator. Specifically, we maintain three
estimators: one for mathematics, one for code, and one shared by logic, science,
simulation, and tabular reasoning. During the first 16 steps, no estimator is
applied and training proceeds with the paired rollout's reward as the baseline,
which fills the trajectory buffer at no additional rollout cost. Each estimator is
then fitted to the trajectories of its domain group and updated online thereafter.

Table~\ref{tab:hidden-state-estimator-hparams} lists the model-specific
estimator hyperparameters.

\begin{table}[H]
    \centering
    \small
    \setlength{\tabcolsep}{4pt}
    \renewcommand{\arraystretch}{1.08}
    \resizebox{\linewidth}{!}{%
    \begin{tabular}{lcc}
        \toprule
        \textbf{Component}
        & \textbf{Qwen3-4B}
        & \textbf{OLMo-3-7B-Instruct-DPO} \\
        \midrule
        Hidden layer
        & \(19\)
        & \(19\) \\

        Prompt hidden-state pooling
        & Last 10 prompt-token mean
        & Last 4 prompt-token mean \\

        Reasoning hidden-state pooling
        & Last 10 reasoning-token mean
        & Last 4 reasoning-token mean \\

        Scalar features
        & 3 entropy statistics
        & 3 entropy statistics \\

        Final estimator input dimension
        & \(5123\)
        & \(8195\) \\
        Regressor
        & StandardScaler + Ridge
        & StandardScaler + Ridge \\

        Ridge penalty
        & \(\alpha=0.01\)
        & \(\alpha=100\) \\

        Estimator routing
        & Math / code / others
        & Math / code / others \\

        Random seed
        & \(42\)
        & \(42\) \\

        Prediction range
        & Clipped to \([0,1]\)
        & Clipped to \([0,1]\) \\
        \bottomrule
    \end{tabular}%
    }
    \caption{
    Hyperparameters of the hidden-state value estimators.
    }
    \label{tab:hidden-state-estimator-hparams}
\end{table}

\subsection{Training details}
\label{appendix:hyperparametersfortraining}
We instantiate our method on two reasoning model backbones, Qwen3-4B
and OLMo-3-7B-Instruct-DPO, the checkpoint released immediately before the RLVR stage of the OLMo3 post-training pipeline. All training is implemented in the \texttt{verl} library. As training data we use GURU~\citep{cheng2025guru}, a curated corpus
of 92K verifiable examples spanning six reasoning domains: mathematics, code, science,
logic, simulation, and tabular reasoning. Each domain is paired with its own verifier, so
that a single training run receives heterogeneous but automatically checkable rewards, and
each batch contains prompts drawn from all six domains. All experiments are run on H100
GPUs. We allocate a fixed compute budget of two days on four GPUs for Qwen3-4B and six days on four GPUs for OLMo-3-7B-Instruct-DPO.

Our main baseline is DAPO~\citep{yu2025dapoopensourcellmreinforcement}, for which we adopt
the implementation of \citet{zheng2025actpaysefficientreinforcement} that improves the
efficiency of DAPO's dynamic sampling, and follow the configurations of the original paper
for all other settings. RLOO uses exactly the same configuration as POISE, except that each rollout's baseline is the other rollout's observed reward. Across all methods, each optimization step consumes the same number of rollouts. Detailed
hyperparameter values are reported in Tables~\ref{tab:hyperparameters_poise},  ~\ref{tab:hyperparameters_greso}.

\begin{table}[h]
\centering
\begin{tabular}{p{0.42\linewidth}p{0.34\linewidth}}
\toprule
\textbf{Hyperparameter} & \textbf{Value} \\
\midrule
algorithm                 & POISE, RLOO \\
training epochs            & 2 \\
train batch size          & 256 \\
samples per prompt        & 2 \\
max prompt length         & 4096 \\
max response length       & 8192 \\
learning rate             & $1 \times 10^{-6}$ \\
clip ratio (low / high)   & 0.2 / 0.28 \\
use KL loss               & False \\
sampling temperature      & 1.0 \\
sampling top-p            & 1.0 \\
\bottomrule
\end{tabular}
\caption{Key training hyperparameters used in POISE, RLOO}
\label{tab:hyperparameters_poise}
\end{table}

\begin{table}[h]
\centering
\begin{tabular}{p{0.42\linewidth}p{0.34\linewidth}}
\toprule
\textbf{Hyperparameter} & \textbf{Value} \\
\midrule
algorithm                 & DAPO \\
training epochs           & 2 \\
train batch size          & 64 \\
samples per prompt        & 8 \\
$p_{\mathrm{easy}}$       & 0.5 \\
$p_{\mathrm{hard}}$       & 0.5 \\
target zero variance      & 0.25 \\
default br size           & 96 \\
GRESO min $p$             & 0.05 \\
GRESO max $p$             & 0.95 \\
$\beta$                   & 1.25 \\
max prompt length         & 4096 \\
max response length       & 8192 \\
learning rate             & $1 \times 10^{-6}$ \\
clip ratio (low / high)   & 0.2 / 0.28 \\
use KL loss               & False \\
sampling temperature      & 1.0 \\
sampling top-p            & 1.0 \\
\bottomrule
\end{tabular}
\caption{Key training hyperparameters used in DAPO with efficient dynamic sampling~\citep{zheng2025actpaysefficientreinforcement}.}
\label{tab:hyperparameters_greso}
\end{table}

\subsection{Evaluation Protocol}
\label{app:eval}
The following section details our evaluation protocol and per domain scoring rules.

\paragraph{Evaluation Protocol.}
\label{app:eval-protocol}
We follow the officially recommended decoding setting with
temperature $0.6$ and top-$p$ $0.95$, and set the maximum response length to 8192 tokens the same as the training setting. All inferences are performed with the vLLM
library~\citep{kwon2023vllm} on a single node equipped with four NVIDIA H100 GPUs.
We evaluate on 18 benchmarks spanning six domains.

\paragraph{Per domain scoring rules.}
All benchmarks are scored by deterministic, execution- or rule-based verifiers; we
use no LLM judge. Each response is reduced to a single answer and mapped to a
binary score, except where noted. When a response contains a \texttt{</think>}
delimiter, only the text after the final delimiter is scored, so that reasoning
traces are never mistaken for the answer.

\begin{itemize}
\item \textbf{Math} (MATH-500, AIME 2024, AIME 2025). The prompt appends
      \emph{``Please output the final answer within \textbackslash boxed\{\}''}
      to the problem statement. We extract the last \texttt{\textbackslash boxed\{\}}
      expression and compare it to the reference answer by symbolic equivalence
      rather than string match, so that algebraically equal forms are accepted.

\item \textbf{Code generation} (HumanEval, MBPP, LiveCodeBench). The prompt asks
      for a complete, self-contained Python solution in a markdown code block;
      LiveCodeBench uses the stdin/stdout formulation. We extract the final code
      block and execute it in a sandbox against the reference test cases, scoring
      $1$ only if every test passes.

\item \textbf{Simulation} (CRUXEval-I, CRUXEval-O, CodeIO). CRUXEval asks the model
      to complete an assertion, predicting an input that yields a given output
      (I) or the output of a given input (O), inside a \texttt{```} python block.
      We concatenate the extracted assertion with the reference function and
      execute it, scoring $1$ if the assertion holds. CodeIO asks for a feasible
      input as a JSON object, which we compare against the reference input.

\item \textbf{Logic} (Zebra Puzzle, ARC-AGI-1). Both prompts require the answer
      inside \texttt{<answer>} \texttt{</answer>} tags: a grid dictionary for
      Zebra Puzzle and a 2D array for ARC-AGI-1. ARC-AGI-1 is scored as exact grid
      match. Zebra Puzzle is the one benchmark scored continuously: we report the
      fraction of correctly filled cells, since a fully binary score is
      uninformative on grids of this size.

\item \textbf{Table reasoning} (HiTab, FinQA, MultiHiertt). The prompt supplies the
      table and, for FinQA, the surrounding text, and asks for the final answer
      within \texttt{\textbackslash boxed\{\}}. We extract the boxed answer and
      compare it to the reference with numeric tolerance and unit normalization;
      multi-part answers are compared element-wise after sorting.

\item \textbf{STEM} (GPQA-Diamond, SuperGPQA). Both are multiple choice with the
      instruction to place only the option letter within
      \texttt{\textbackslash boxed\{\}}. We extract that letter and score exact
      match; a missing or malformed box scores $0$.

\item \textbf{Out-of-domain} (IFEval, LiveBench-Reasoning, LiveBench-Language,
      LiveBench-Data-Analysis). IFEval carries no added instruction beyond its own
      constraints; we run the official verifiable-instruction checkers and score
      $1$ only if \emph{all} constraints of a prompt are satisfied. LiveBench
      prompts specify a task-appropriate answer format, and we apply the official
      per-task scorer for each split.
\end{itemize}

\subsection{Prompt Templates}
\label{appendix:prompt}
Each domain uses a fixed single-turn template that encodes the answer format its
verifier expects; the templates are identical at training and inference time and
across models, and each rendered prompt is then wrapped in the evaluated model's
own chat template. The placeholder \texttt{\{problem\}} is replaced with the
problem statement.

\begin{tcolorbox}[
    breakable,
    colback=gray!5, colframe=gray!50,
    title={\textbf{Mathematical reasoning} (MATH-500, AIME 2025)},
    fonttitle=\bfseries, boxrule=0.5pt, arc=2pt,
]
\small
\texttt{\{problem\}} Please output the final answer within
\textbackslash boxed\{\}.
\end{tcolorbox}

\begin{tcolorbox}[
    breakable,
    colback=gray!5, colframe=gray!50,
    title={\textbf{Code generation} (HumanEval, MBPP)},
    fonttitle=\bfseries, boxrule=0.5pt, arc=2pt,
]
\small
Write a complete, self-contained Python solution to the following problem. Your
solution must include all necessary imports and the full function definition
including the signature exactly as specified. Do not modify the function
signature or docstring.

\vspace{0.5em}
\texttt{\{problem\}}
\vspace{0.5em}

\textit{(MBPP additionally lists the assertion test cases the solution must pass,
followed by:)} Please do not include the test cases in your solution.
\end{tcolorbox}

\begin{tcolorbox}[
    breakable,
    colback=gray!5, colframe=gray!50,
    title={\textbf{Code generation} (LiveCodeBench)},
    fonttitle=\bfseries, boxrule=0.5pt, arc=2pt,
]
\small
You are an expert Python programmer. You will be given a question (problem
specification) and will generate a correct Python program that matches the
specification and passes all tests.

\vspace{0.5em}
Below is the question:
\vspace{0.5em}

\texttt{\{problem\}}
\vspace{0.5em}

Read the inputs from stdin solve the problem and write the answer to stdout (do
not directly test on the sample inputs). Enclose your code within delimiters.
\end{tcolorbox}

\begin{tcolorbox}[
    breakable,
    colback=gray!5, colframe=gray!50,
    title={\textbf{Simulation} (CRUXEval-O)},
    fonttitle=\bfseries, boxrule=0.5pt, arc=2pt,
]
\small
You will be given a Python function and an assertion containing an input to the
function. Complete the assertion with a literal (no unsimplified expressions, no
function calls) containing the output when executing the provided code on the
given input, even if the function is incorrect or incomplete. Do NOT output any
extra information. Execute the program step by step before arriving at an answer,
and provide the full assertion with the correct output in
\verb|```python| and \verb|```| tags, following the examples.

\vspace{0.5em}
\texttt{\{problem\}}
\end{tcolorbox}

\noindent CRUXEval-I inverts the direction, asking for any input satisfying
\texttt{f(??) == output}; CodeIO asks the model to predict a feasible input as a
JSON object without writing code.

\begin{tcolorbox}[
    breakable,
    colback=gray!5, colframe=gray!50,
    title={\textbf{Logic} (ARC-AGI-1)},
    fonttitle=\bfseries, boxrule=0.5pt, arc=2pt,
]
\small
You are a world-class puzzle solver with exceptional pattern recognition skills.
Your task is to analyze puzzles, spot patterns, and provide direct solutions.
Given input-output grid pairs as reference examples, carefully observe the
patterns to predict the output grid for new test input.

\vspace{0.5em}
\texttt{\{problem\}}
\vspace{0.5em}

What is the output grid? Please put your answer within
\texttt{<answer>} and \texttt{</answer>} tags, your final answer should be only
the output grid (2d array).
\end{tcolorbox}

\noindent Zebra Puzzle uses the same \texttt{<answer>} convention, but asks for
the completed grid as a dictionary with \texttt{header} and \texttt{rows} keys.

\begin{tcolorbox}[
    breakable,
    colback=gray!5, colframe=gray!50,
    title={\textbf{Table reasoning} (HiTab, FinQA, MultiHiertt)},
    fonttitle=\bfseries, boxrule=0.5pt, arc=2pt,
]
\small
You are given one or more tables. Use the information in the table to answer the
following question.

\vspace{0.5em}
\texttt{\{table\}}
\vspace{0.5em}

The question is:

\texttt{\{problem\}}

Please output the final answer within \textbackslash boxed\{\}.
\textit{(HiTab adds: If there are multiple answers, please output them separated
by \texttt{|}.)}
\end{tcolorbox}

\begin{tcolorbox}[
    breakable,
    colback=gray!5, colframe=gray!50,
    title={\textbf{STEM} (GPQA-Diamond, SuperGPQA)},
    fonttitle=\bfseries, boxrule=0.5pt, arc=2pt,
]
\small
\texttt{\{problem\}}

\texttt{\{options\}}
\vspace{0.5em}

Please reason step by step, and put your final answer option within
\textbackslash boxed\{\}. Only put the letter in the box, e.g.
\textbackslash boxed\{A\}. There is only one correct answer.
\end{tcolorbox}

\noindent IFEval and LiveBench are used verbatim as released: IFEval prompts
already carry their own verifiable constraints, and each LiveBench task specifies
its own answer format (for example, \emph{``put your answer in \textbf{bold} as a
single integer''}). We add no instruction of our own to either.

\clearpage

\section{Instruction-Following Experiments}
\label{app:instruction_following}

\begin{table}[!h]
\centering
\small
\setlength{\tabcolsep}{7pt}
\renewcommand{\arraystretch}{1.08}
\caption{
Prompt-level loose accuracy (\%) after instruction-following training.
The best and second-best results for each backbone are bolded and
underlined, respectively.
}
\label{tab:instruction_following}
\begin{tabular}{llccc}
\toprule
Backbone & Method & IFEval & IFBench & Average \\
\midrule

\multirow{4}{*}{Qwen3-4B}
& Base
& 83.0 & 34.3 & 58.7 \\
& RLOO ($m=2$)
& 81.7 & 35.0 & 58.4 \\
& DAPO ($m=8$)
& \underline{83.9} & \underline{37.7} & \underline{60.8} \\
& POISE ($m=2$)
& \textbf{84.8} & \textbf{39.7} & \textbf{62.3} \\

\midrule

\multirow{4}{*}{
    \shortstack[l]{DeepSeek-R1-Distill-\\Qwen-1.5B}
}
& Base
& 36.0 & 10.3 & 23.2 \\
& RLOO ($m=2$)
& 52.9 & 15.3 & 34.1 \\
& DAPO ($m=8$)
& \underline{54.0} & \underline{16.7} & \underline{35.3} \\
& POISE ($m=2$)
& \textbf{56.6} & \textbf{17.3} & \textbf{36.9} \\

\bottomrule
\end{tabular}
\end{table}

\paragraph{Setup.}
We evaluate POISE in instruction-following RL using the IF-RLVR training split released with IFBench~\citep{pyatkin2026generalizing}. We use the same optimization setup as in our main experiments and compare POISE and RLOO with two rollouts per prompt ($m=2$) against DAPO with eight rollouts per prompt ($m=8$). We evaluate on the IFBench test set, which contains constraint types not observed during training, and on IFEval~\citep{zhou2023ifeval}. Following the IFBench evaluation protocol, we use greedy decoding with temperature 0 and report prompt-level loose accuracy for both benchmarks.

\paragraph{Results.}
Table~\ref{tab:instruction_following} presents the results. POISE achieves the highest accuracy on both evaluation sets and across both backbones. Compared with DAPO, POISE improves IFBench accuracy by 2.0 percentage points on Qwen3-4B and 0.6 points on DeepSeek-R1-Distill-Qwen-1.5B. It also improves IFEval accuracy by 0.9 and 2.6 points, respectively. The consistent gains across in-domain instruction-following training and unseen evaluation constraints provide further evidence that POISE is applicable beyond mathematical reasoning.

\clearpage
\section{Training Dynamics}
\label{app:training_dynamics}

We provide additional training dynamics for POISE on
Qwen3-4B and Olmo-3-7B-Instruct-DPO. In the main text, we report aggregate
evidence that the internal state value estimator remains stable during training.
Here, we expand the analysis by tracking six quantities over policy updates:
the batch reward, advantage, token-level entropy, the value-estimation
error against the online target, gradient norm and response length.

For each training step \(t\), let \(R_t\) denote the verifier reward of the
sampled rollouts, \(b_t=g_f(\phi_t)\) the predicted baseline, and
\(A_t=R_t-b_t\) the resulting advantage. We report the batch reward
\(\bar R_t=\mathbb{E}_{\mathrm{batch}}[R_t]\), the online target error
\(\mathrm{MAE}_t=\mathbb{E}_{\mathrm{batch}}[|b_t-\widehat V_t(x)|]\).
Here, \(\widehat V_t(x)\) is the empirical online target estimated from rollouts
sampled from the current checkpoint policy.

\paragraph{Qwen3-4B.}
Figure~\ref{fig:training-dynamics-qwen4b} shows that the mean reward rises
rapidly from approximately 0.45 to 0.65 during the first 100 steps and remains
near that level thereafter. The mean advantage stays within a relatively narrow
negative range without sustained drift. Despite the substantial change in
policy reward, the online-target MAE remains bounded between approximately 0.28
and 0.33, with only a modest upward trend, indicating that the estimator
continues to track the evolving policy. 

Policy entropy increases throughout training, with a sharper rise near the end,
while response length drops rapidly from approximately 6,500 to 4,500 tokens
before stabilizing. Thus, the reward improvement is not driven by increasing
response length or collapsing policy entropy, and both the estimator and policy
optimization remain stable after the initial adaptation phase.


\begin{figure}[H]
    \centering
    \begin{subfigure}[t]{0.32\linewidth}
        \centering
        \includegraphics[width=\linewidth]{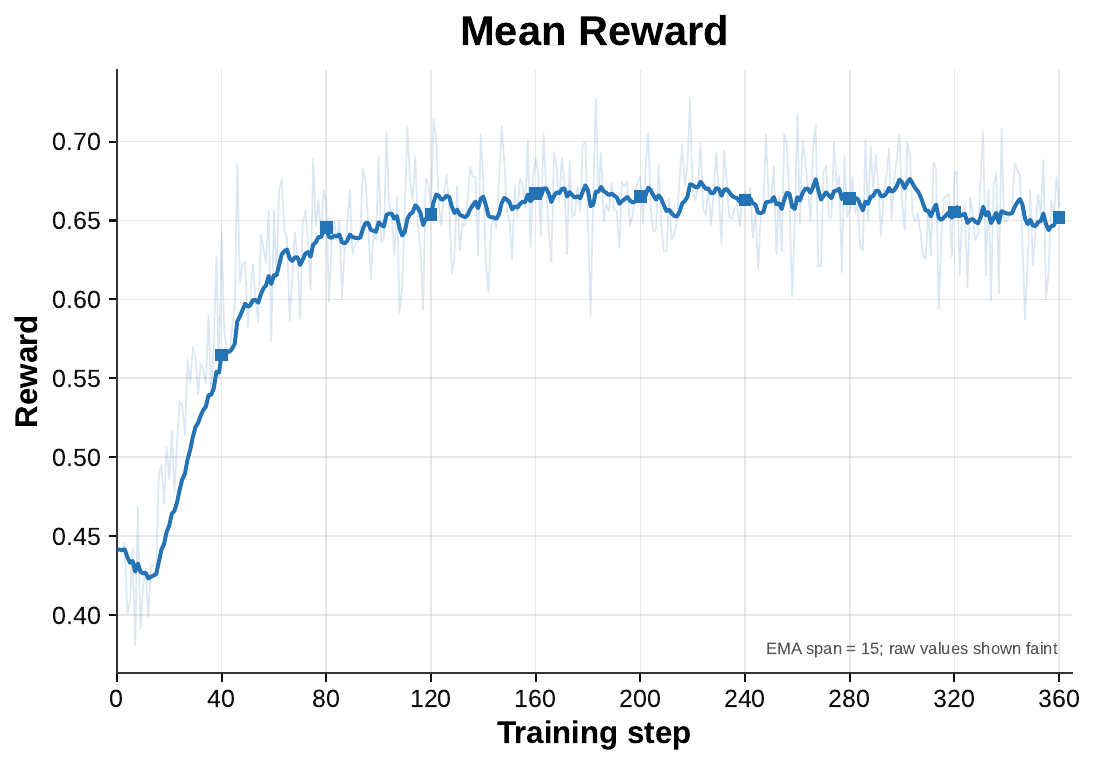}
        \caption{Batch reward}
    \end{subfigure}
    \hfill
    \begin{subfigure}[t]{0.32\linewidth}
        \centering
        \includegraphics[width=\linewidth]{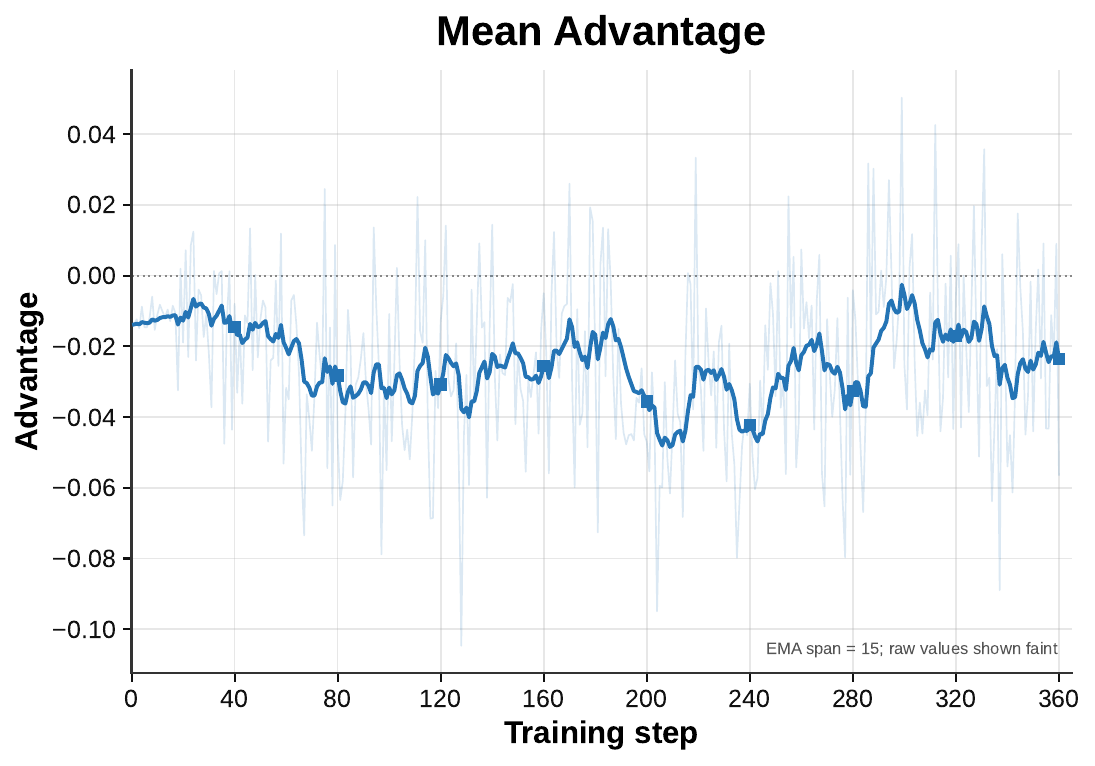}
        \caption{Advantage}
    \end{subfigure}
    \hfill
    \begin{subfigure}[t]{0.32\linewidth}
        \centering
        \includegraphics[width=\linewidth]{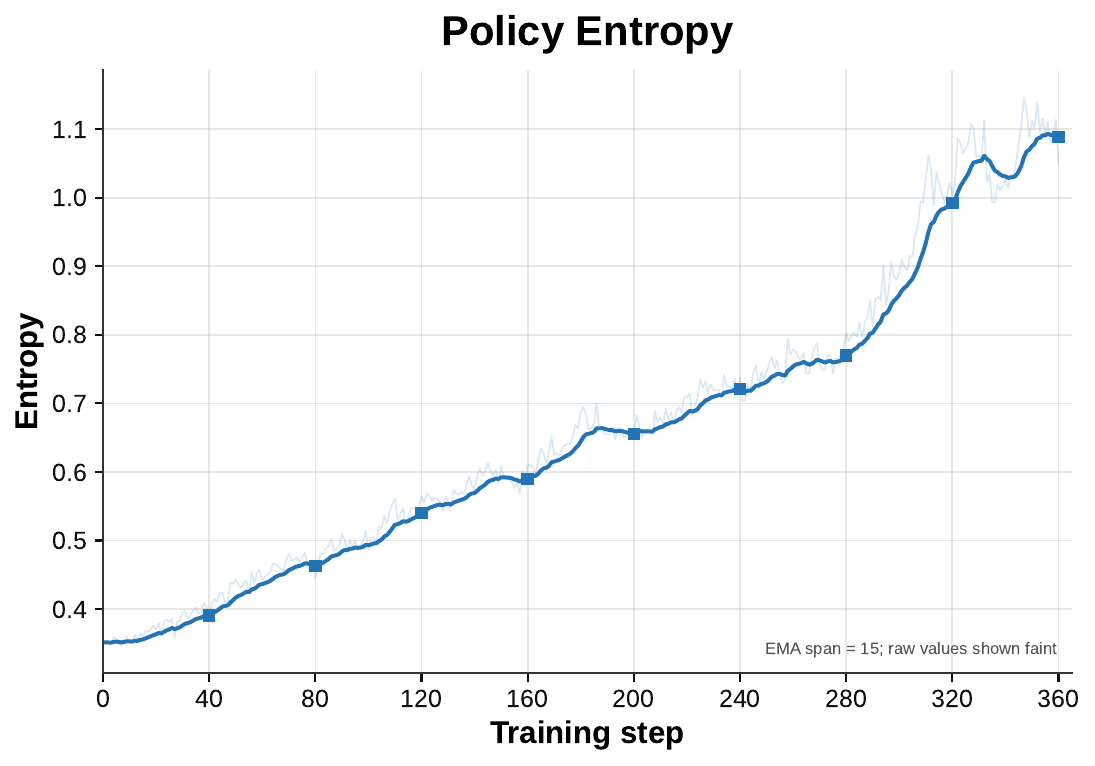}
        \caption{Entropy}
    \end{subfigure}

    \vspace{0.75em}

    \begin{subfigure}[t]{0.32\linewidth}
        \centering
        \includegraphics[width=\linewidth]{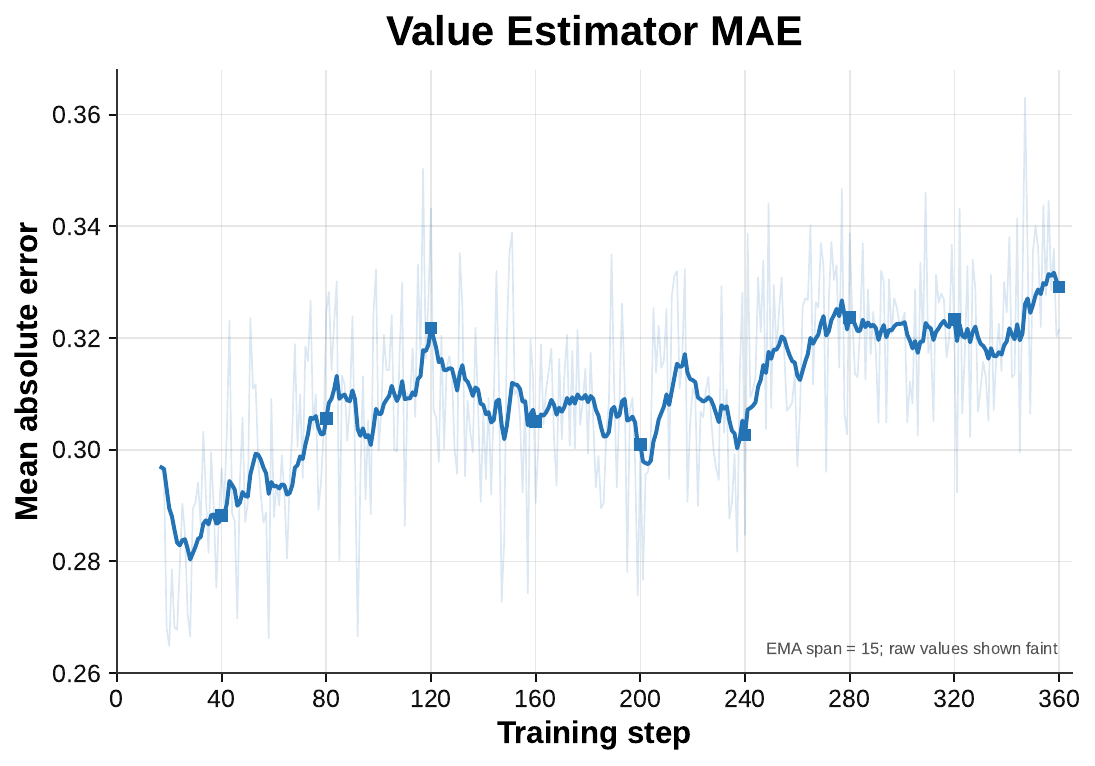}
        \caption{Online target MAE}
    \end{subfigure}
    \hfill
    \begin{subfigure}[t]{0.32\linewidth}
        \centering
        \includegraphics[width=\linewidth]{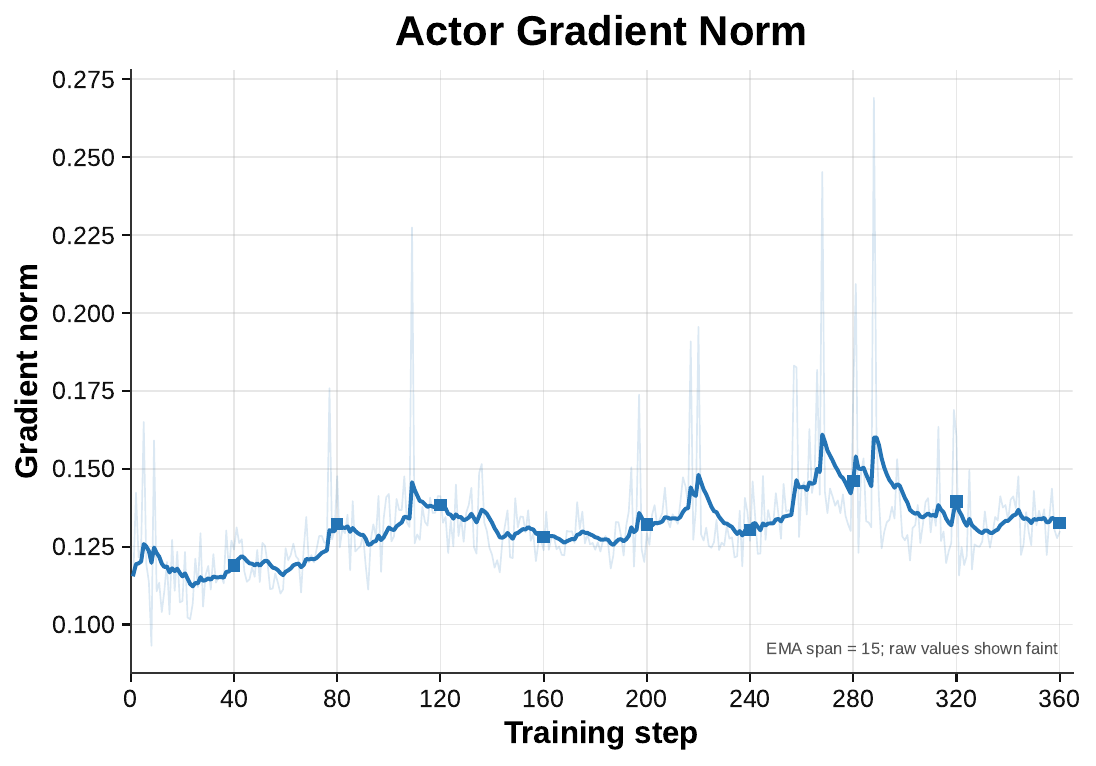}
        \caption{Gradient norm}
    \end{subfigure}
    \begin{subfigure}[t]{0.32\linewidth}
        \centering
        \includegraphics[width=\linewidth]{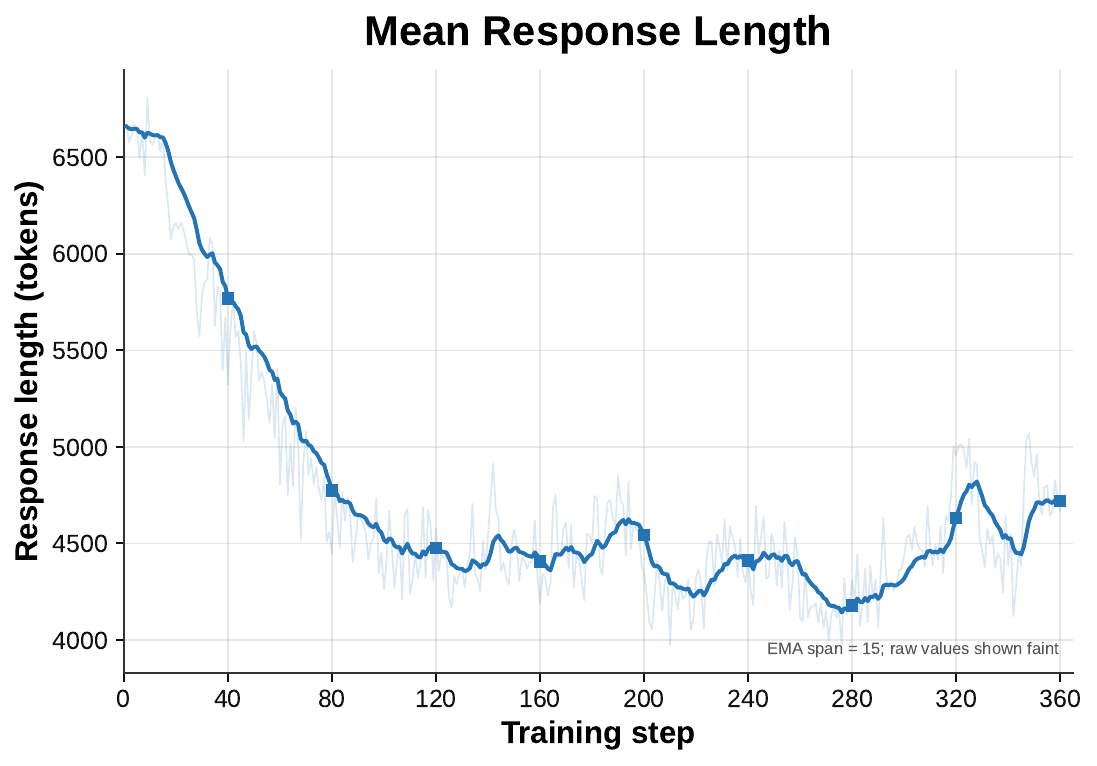}
        \caption{Response length}
    \end{subfigure}
    \caption{
    Training dynamics of POISE on Qwen3-4B.
    }
    \label{fig:training-dynamics-qwen4b}
\end{figure}

\paragraph{OLMo-3-7B-Instruct-DPO.}
Figure~\ref{fig:training-dynamics-olmo3} shows a more gradual improvement: the
mean reward increases steadily from approximately 0.28 to 0.60 over training.
The mean advantage settles into a narrow negative range after the initial
steps, while the online-target MAE rises during early adaptation and then
remains around 0.34--0.36 for most of training. The actor gradient norm likewise
stabilizes after an initial decline, aside from an isolated spike late in
training.

Unlike Qwen3-4B, policy entropy evolves non-monotonically, increasing modestly
during the first half of training before declining later. Response length also
fluctuates after its initial increase but exhibits no sustained growth. Despite
these model-specific changes in exploration and generation behavior, reward
continues to improve without a corresponding divergence in estimation error or
gradient magnitude.

\begin{figure}[H]
    \centering
    \begin{subfigure}[t]{0.32\linewidth}
        \centering
        \includegraphics[width=\linewidth]{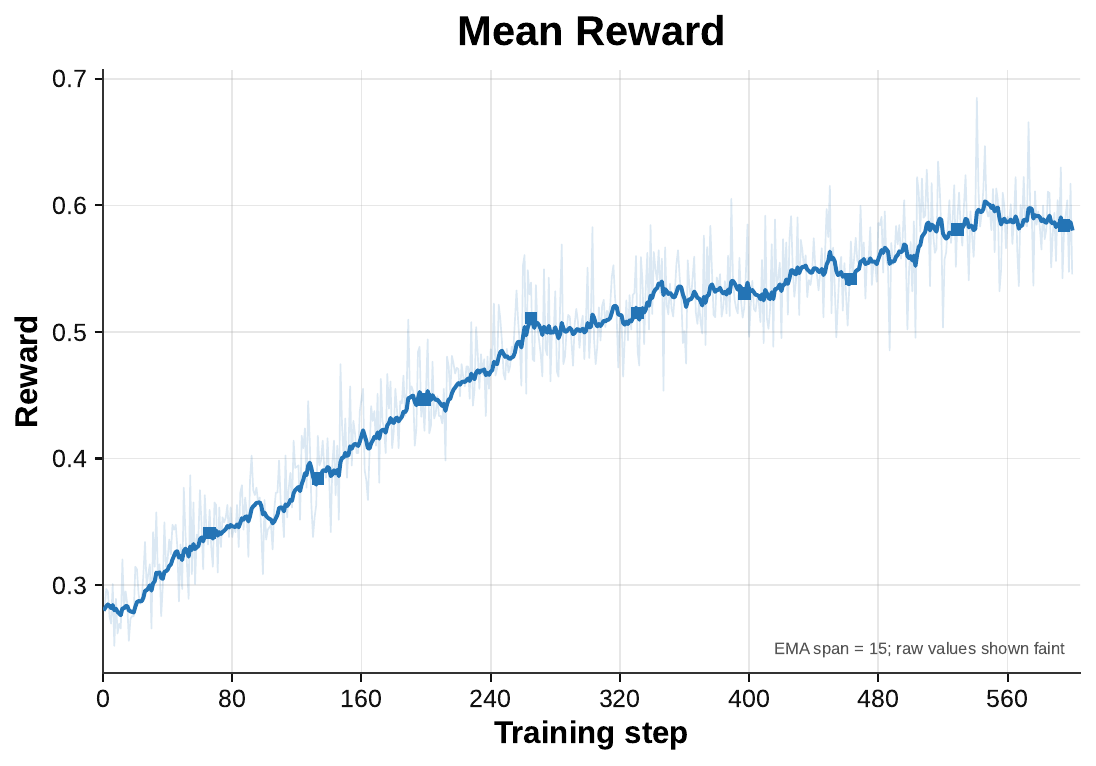}
        \caption{Batch reward}
    \end{subfigure}
    \hfill
    \begin{subfigure}[t]{0.32\linewidth}
        \centering
        \includegraphics[width=\linewidth]{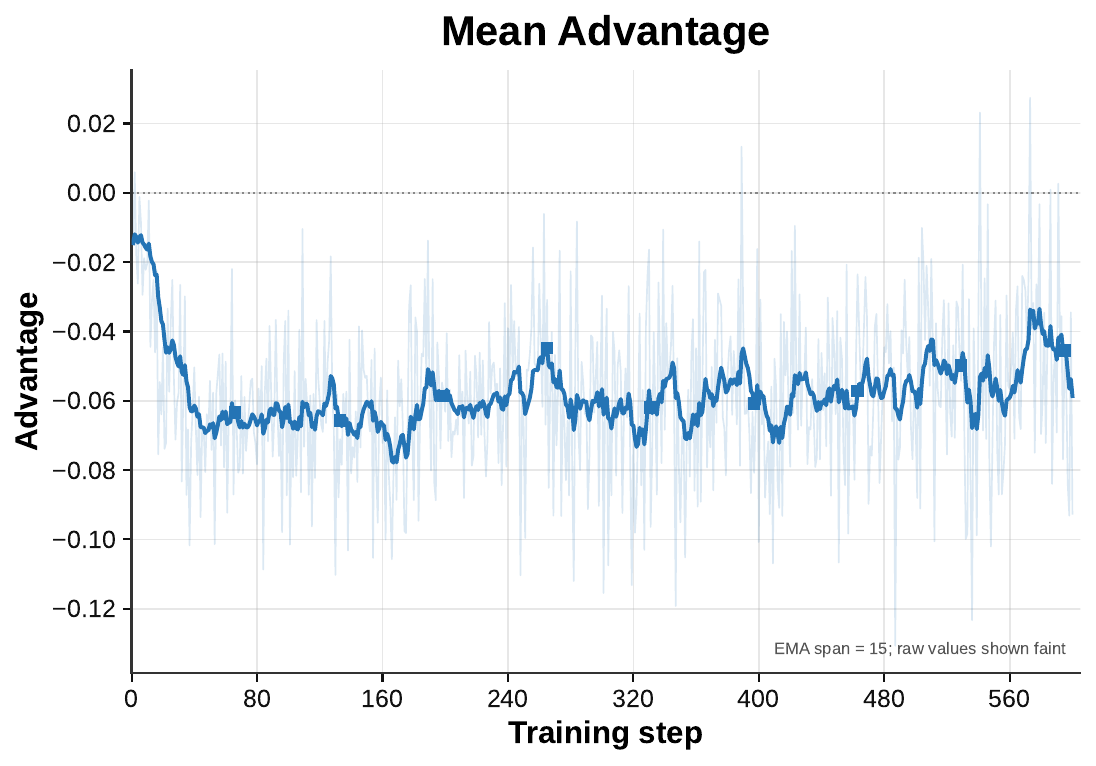}
        \caption{Advantage}
    \end{subfigure}
    \hfill
    \begin{subfigure}[t]{0.32\linewidth}
        \centering
        \includegraphics[width=\linewidth]{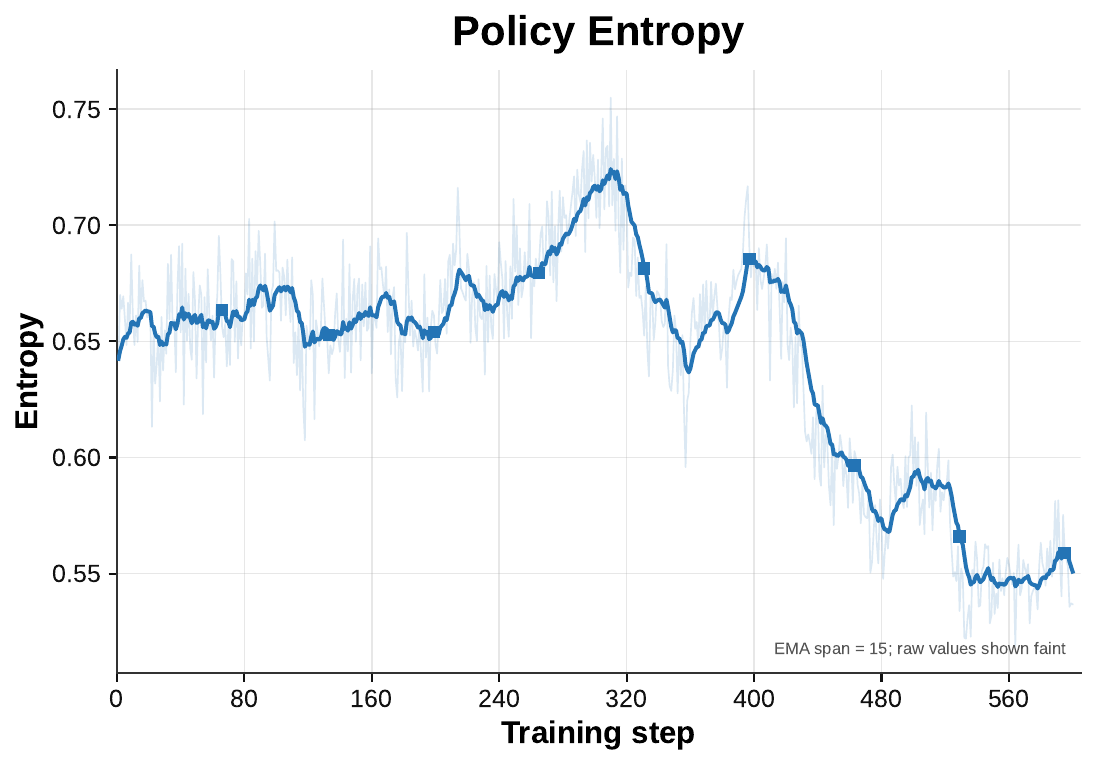}
        \caption{Entropy}
    \end{subfigure}

    \vspace{0.75em}

    \begin{subfigure}[t]{0.32\linewidth}
        \centering
        \includegraphics[width=\linewidth]{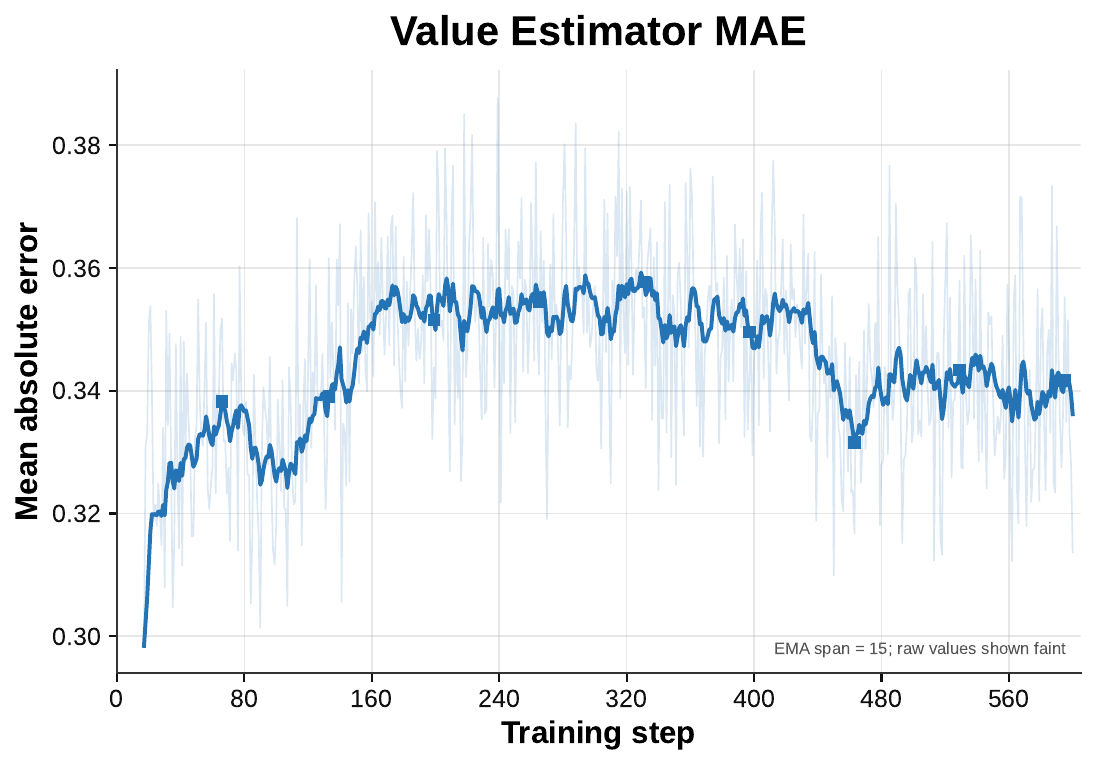}
        \caption{Online target MAE}
    \end{subfigure}
    \hfill
    \begin{subfigure}[t]{0.32\linewidth}
        \centering
        \includegraphics[width=\linewidth]{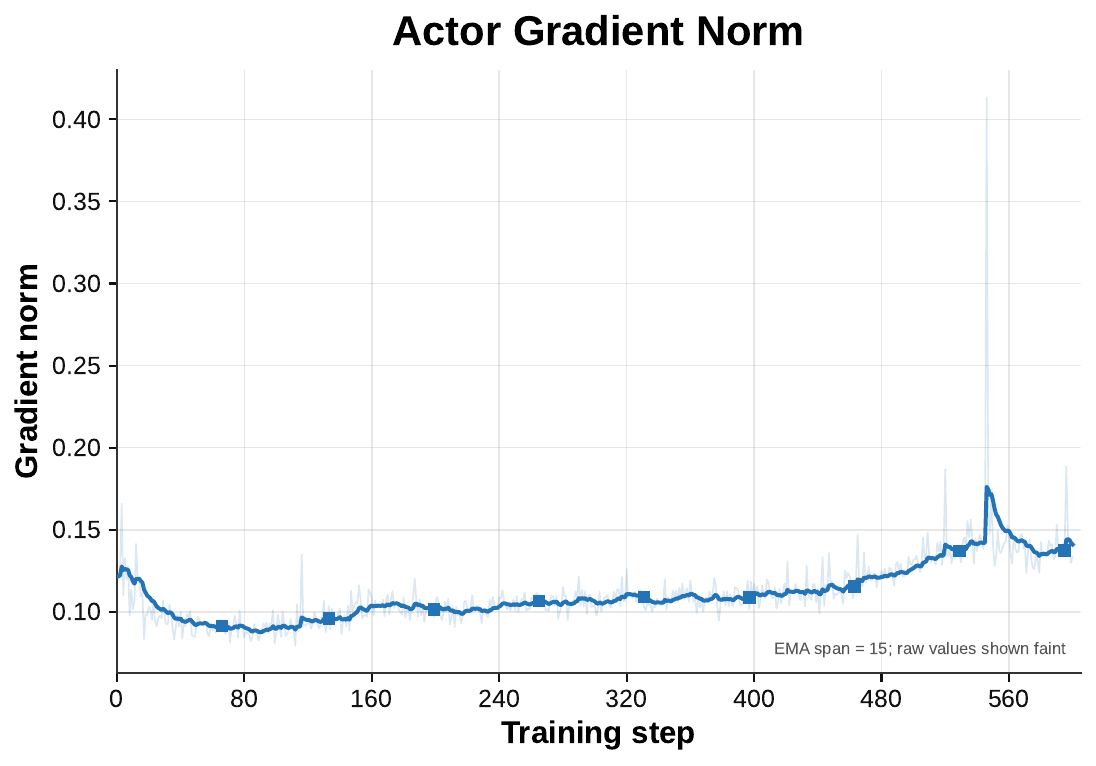}
        \caption{Gradient norm}
    \end{subfigure}
    \begin{subfigure}[t]{0.32\linewidth}
        \centering
        \includegraphics[width=\linewidth]{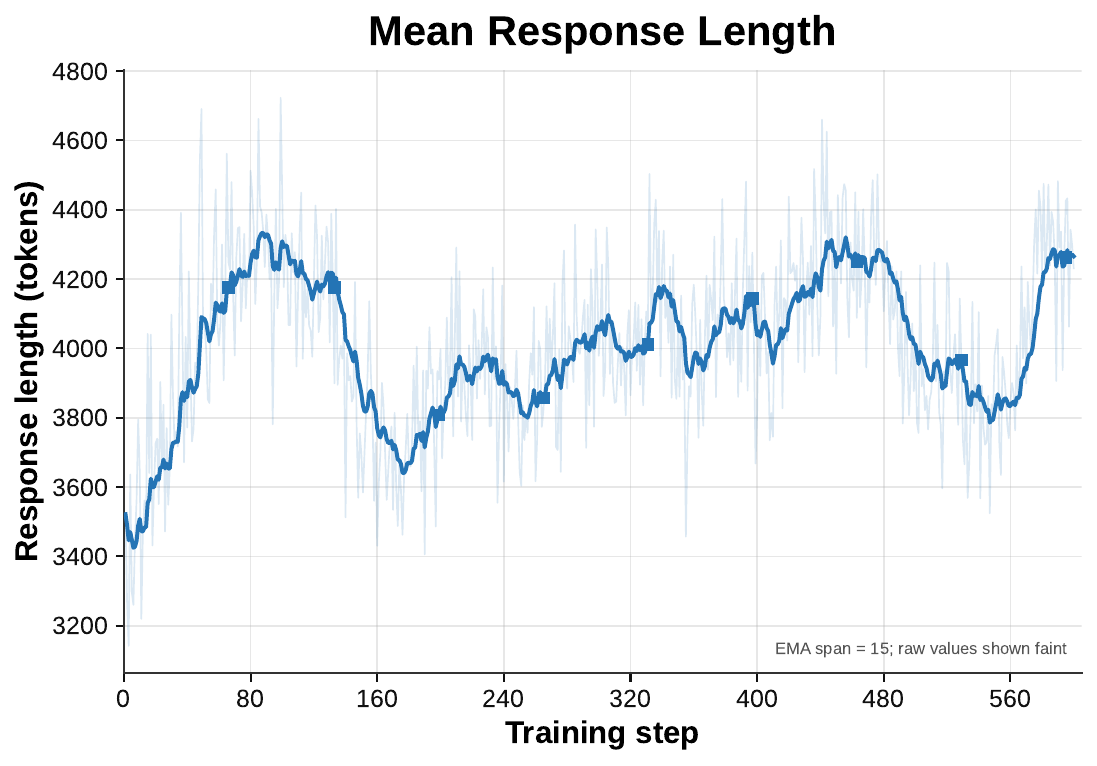}
        \caption{Response length}
    \end{subfigure}
    \caption{
    Training dynamics of POISE on Olmo-3-7B-Instruct-DPO.
    }
    \label{fig:training-dynamics-olmo3}
\end{figure}

\clearpage

\section{Comparisons with Policy-Model-Scale Critics}
\label{app:critic_comparison}
We compare POISE with policy-model-scale critics in two settings. The first
examines held-out value prediction and supports the preliminary analysis in
Section~\ref{sec:Estimating_Baseline_Values}. The second directly compares
POISE with a token-level critic in end-to-end RL.

\subsection{Critic for Held-Out Value Prediction}
\label{app:critic}
\begin{table*}[h]
\centering
\small
\caption{Hyperparameters for the held-out prompt-level scalar critic.}
\label{tab:heldout-critic-hparams}
\begin{tabular}{p{0.30\textwidth}p{0.64\textwidth}}
\toprule
Hyperparameter & Value \\
\midrule
Trainable parameters & Full critic backbone and scalar head  \\
Objective & Bernoulli cross-entropy (\texttt{BCEWithLogitsLoss}) \\
Optimizer & AdamW \\
Backbone/head learning rate & $10^{-5}$ / $10^{-4}$ \\
Learning-rate schedule & None \\
Weight decay & 0.01 (bias and normalization parameters excluded) \\
Batch size & 8 / 16 \\
Gradient clipping & Global norm 1.0 \\
Precision & FP32 (BF16 autocast) \\
\bottomrule
\end{tabular}
\end{table*}

For the held-out value-prediction comparison, we implement a prompt-level
scalar critic following the sequence-level critic design of
SPPO~\citep{wang2026spposequencelevelppolonghorizon}. For each policy backbone,
the critic is initialized from the same pretrained checkpoint as the policy and
augmented with a bias-free scalar value head. Given a prompt $x$, it processes
the prompt only, extracts the hidden state at the final non-padding prompt
token, and predicts

\begin{equation}
z_\phi(x) = w^\top h_\phi(x),
\qquad
V_\phi(x) = \sigma\!\left(z_\phi(x)\right).
\end{equation}

We optimize both the scalar head and the full critic backbone, except for GPT-OSS 120B where we use a LoRA adapter instead due to memory constraints. Thus, although the critic produces a single prompt-level scalar, it trains a policy-model-scale representation rather than fitting only a small output head.

For each training prompt, we sample $K_{\mathrm{train}}$ responses from the
target policy and obtain binary rewards $r^{(j)}=R(x,y^{(j)})$. The
implementation minimizes Bernoulli cross-entropy over the individual
prompt--rollout pairs,

\begin{equation}
\mathcal{L}_{\mathrm{critic}}(\phi)
=
-\frac{1}{K_{\mathrm{train}}}
\sum_{j=1}^{K_{\mathrm{train}}}
\left[
r^{(j)}\log V_\phi(x)
+ \left(1-r^{(j)}\right)\log\left(1-V_\phi(x)\right)
\right].
\end{equation}

Because all rollouts for a prompt share the same prediction, this objective is
exactly equivalent to soft-label cross-entropy with

\begin{equation}
\widehat V(x)
=
\frac{1}{K_{\mathrm{train}}}
\sum_{j=1}^{K_{\mathrm{train}}}
R\!\left(x,y^{(j)}\right).
\end{equation}

We train this critic and the internal-state estimator on the same
reward-labeled examples and evaluate both against empirical Avg@8 values on
held-out prompts using MAE and Pearson correlation. The critic is used only as
an analysis baseline and does not participate in POISE training.

\paragraph{Comparisons on Multiple Domains and Models}
\label{app:generalizability}

We evaluate whether our internal-state value estimator generalizes to different
verifiable-reward domains and policy backbones. We consider five datasets
spanning math reasoning, coding, tool use, and instruction following:
DAPO-Math~\citep{yu2025dapoopensourcellmreinforcement},
DeepScaleR~\citep{deepscaler2025}, AceCoder~\citep{acecoder},
ToolDial~\citep{shim2025tooldial}, and IF-RLVR~\citep{pyatkin2026generalizing}. We also
evaluate four policy backbones: Qwen3-4B, DeepSeek-R1-Distill-Qwen-1.5B, DeepSeek-R1-Distill-Qwen-7B, and GPT-OSS 120B.

For each model--dataset pair, we train the internal-state estimator and a
separately trained critic on the same reward-labeled rollout data. To make the
comparison consistent across settings, we use 4,096 training examples for each
estimator, matching the number of examples stored in the value-estimator buffer
used by the POISE algorithm. We then evaluate both estimators against empirical
Avg@8 values on held-out prompts. For evaluation, we use 1,024 held-out prompts
for all datasets except ToolDial, for which we use 13,492 held-out examples.

Table~\ref{tab:generalization-full} shows that our internal-state estimator
outperforms the critic in most settings, reducing MAE in 16 of 20
model--dataset pairs and improving Pearson correlation in 18.
On Qwen3-4B, the estimator improves both metrics across all five datasets,
with particularly large gains outside the original math setting:
Pearson correlation increases from 0.056 to 0.612 on AceCoder and
from 0.150 to 0.642 on IF-RLVR. The trend also holds for larger models.
On both DeepSeek-R1-Distill-Qwen-7B and GPT-OSS-120B, the estimator
improves both metrics on four of the five datasets. Results are more
mixed but still favorable overall on DeepSeek-R1-Distill-Qwen-1.5B:
the estimator improves both metrics on DAPO-Math, DeepScaleR, and
ToolDial, while the critic performs better on AceCoder and achieves
lower MAE on IF-RLVR.

Overall, these results show that value-relevant information
can be recovered from policy internal states across different domains
and model scales, motivating our use of these signals for baseline
estimation within the POISE training algorithm (\S~\ref{sec:cross_rollout}).

\begin{table*}[t]
\centering
\scriptsize
\caption{
Full value-prediction results across policy backbones and verifiable-reward domains.
Where available, we compare our internal-state estimator against a separately trained critic and report MAE and Pearson correlation \(r\).
}
\label{tab:generalization-full}
\setlength{\tabcolsep}{2.5pt}
\renewcommand{\arraystretch}{1.05}
\resizebox{\textwidth}{!}{%
\begin{tabular}{lllcccc}
\toprule
\multirow{2}{*}{Policy model}
& \multirow{2}{*}{Domain}
& \multirow{2}{*}{Dataset}
& \multicolumn{2}{c}{Critic}
& \multicolumn{2}{c}{Internal-state estimator} \\
\cmidrule(lr){4-5} \cmidrule(lr){6-7}
& &
& MAE $\downarrow$ & $r$ $\uparrow$
& MAE $\downarrow$ & $r$ $\uparrow$ \\
\midrule

\multirow{5}{*}{Qwen3-4B}
& \multirow{2}{*}{Math reasoning}
& DAPO-Math
& 0.262 & 0.676
& \textbf{0.141} & \textbf{0.870} \\

&
& DeepScaleR
& 0.393 & 0.384
& \textbf{0.231} & \textbf{0.609} \\

& Coding
& AceCoder
& 0.499 & 0.056
& \textbf{0.234} & \textbf{0.612} \\

& Tool calling
& ToolDial
& 0.303 & 0.440
& \textbf{0.188} & \textbf{0.840} \\

& Instruction following
& IF-RLVR
& 0.350 & 0.150
& \textbf{0.195} & \textbf{0.642} \\

\midrule
\multirow{5}{*}{\makecell[l]{DeepSeek-R1-Distill-\\Qwen-1.5B}}
& \multirow{2}{*}{Math reasoning}
& DAPO-Math
& 0.127 & 0.723
& \textbf{0.123} & \textbf{0.834} \\

&
& DeepScaleR
& 0.251 & 0.586
& \textbf{0.151} & \textbf{0.829} \\

& Coding
& AceCoder
& \textbf{0.135} & \textbf{0.706}
& 0.196 & 0.580 \\

& Tool calling
& ToolDial
& 0.201 & 0.337
& \textbf{0.122} & \textbf{0.672} \\

& Instruction following
& IF-RLVR
& \textbf{0.101} & 0.545
& 0.171 & \textbf{0.557} \\

\midrule
\multirow{5}{*}{\makecell[l]{DeepSeek-R1-Distill-\\Qwen-7B}}
& \multirow{2}{*}{Math reasoning}
& DAPO-Math
& 0.300 & 0.441
& \textbf{0.191} & \textbf{0.784} \\

&
& DeepScaleR
& 0.270 & 0.457
& \textbf{0.164} & \textbf{0.814} \\

& Coding
& AceCoder
& 0.229 & 0.691
& \textbf{0.173} & \textbf{0.804} \\

& Tool calling
& ToolDial
& 0.194 & 0.663
& \textbf{0.153} & \textbf{0.842} \\

& Instruction following
& IF-RLVR
& \textbf{0.155} & \textbf{0.716}
& 0.191 & 0.596 \\

  \midrule
  \multirow{5}{*}{GPT-OSS-120B}
  & \multirow{2}{*}{Math reasoning}
  & DAPO-Math
  & 0.241 & 0.495
  & \textbf{0.206} & \textbf{0.661} \\

  &
  & DeepScaleR
  & 0.219 & 0.531
  & \textbf{0.166} & \textbf{0.806} \\

  & Coding
  & AceCoder
  & \textbf{0.222} & 0.602
  & 0.233 & \textbf{0.623} \\

  & Tool calling
  & ToolDial
  & 0.191 & 0.752
  & \textbf{0.179} & \textbf{0.836} \\

  & Instruction following
  & IF-RLVR
  & 0.281 & 0.341
  & \textbf{0.190} & \textbf{0.671} \\

  \bottomrule\end{tabular}%
}
\end{table*}

\subsection{End-to-end comparison with a token-level PPO critic}
\label{app:ppo_multidomain}

The preceding held-out comparison uses a prompt-level critic to evaluate
value prediction. We now compare complete RL systems in the main
multi-domain setting. We run PPO with a separate, policy-model-scale
token-level critic on Qwen3-4B and OLMo-3-7B-Instruct-DPO. The policy
initialization, six-domain training mixture, actor-side optimization
settings, per-update rollout budget, validation-based checkpoint selection,
and evaluation protocol follow the main experiments. PPO samples
one rollout per prompt ($m=1$), while POISE samples two ($m=2$). At the
same per-update trajectory budget, PPO can therefore cover twice as many
distinct prompts.

The PPO critic is initialized from the actor checkpoint, predicts a value
for each generated prefix, and is trained as a full-parameter model using
clipped value regression. Following the setup in Open-Reasoner-Zero~\citep{hu2026open} and Deepseek-R1~\citep{deepseek_r1}, PPO computes advantages using generalized
advantage estimation with $\gamma=\lambda=1$. 

\begin{table*}[t]
\centering
\caption{Token-level PPO critic configuration \textbf{(a)} and domain-level
performance \textbf{(b)}. Domain scores are averaged over the corresponding benchmarks
in Table~\ref{tab:benchmark_results}; Overall averages all 17 benchmarks.
PPO uses $m=1$ and POISE uses $m=2$. OLMo-3-7B abbreviates
OLMo-3-7B-Instruct-DPO.}
\label{tab:critic-rl-hparams}
\label{tab:ppo_domain_comparison}

\begin{minipage}[t]{0.48\textwidth}
\vspace{0pt}
\footnotesize
\textbf{(a) Critic configuration}\par\smallskip
\setlength{\tabcolsep}{3pt}
\begin{tabular*}{\linewidth}{@{\extracolsep{\fill}}ll@{}}
\toprule
Hyperparameter & Value \\
\midrule
Initialization & Actor checkpoint \\
Trainable parameters & Full critic \\
Critic learning rate & $10^{-5}$ \\
Critic value clipping & 0.5 \\
Critic objective & Clipped value regression \\
GAE $(\gamma,\lambda)$ & $(1.0,1.0)$ \\
\bottomrule
\end{tabular*}
\end{minipage}
\hfill
\begin{minipage}[t]{0.50\textwidth}
\vspace{0pt}
\footnotesize
\textbf{(b) Domain-level performance}\par\smallskip
\setlength{\tabcolsep}{2.5pt}
\begin{tabular*}{\linewidth}{@{\extracolsep{\fill}}lrrrr@{}}
\toprule
& \multicolumn{2}{c}{\textbf{Qwen3-4B}}
& \multicolumn{2}{c}{\textbf{OLMo-3-7B}} \\
\cmidrule(lr){2-3}\cmidrule(lr){4-5}
Domain & PPO & POISE & PPO & POISE \\
\midrule
Code                  & \textbf{73.05} & 71.75          & 62.46          & \textbf{63.96} \\
Logic                 & 13.50          & \textbf{22.28} & \textbf{18.68} & 17.08 \\
Math                  & \textbf{69.28} & 67.72          & 62.22          & \textbf{64.02} \\
Simulation            & \textbf{65.33} & 64.62          & \textbf{59.63} & 58.63 \\
Science               & 49.28          & \textbf{50.18} & 34.60          & \textbf{39.69} \\
Tabular reasoning     & 55.71          & \textbf{57.09} & \textbf{43.28} & 42.84 \\
\addlinespace[2pt]
LiveBench             & 45.18          & \textbf{45.35} & \textbf{33.62} & 33.54 \\
Instruction following & 81.52          & \textbf{81.70} & \textbf{77.82} & 77.26 \\
\midrule
\textbf{Overall}      & 57.24          & \textbf{58.10} & 49.33          & \textbf{49.92} \\
\bottomrule
\end{tabular*}
\end{minipage}
\end{table*}

Table~\ref{tab:ppo_domain_comparison} (right) shows that POISE achieves
higher overall performance than token-level PPO on both policy backbones.
PPO performs better in some individual domains, but POISE leads on the
overall benchmark average. This holds even though PPO's $m=1$ allocation
allows it to cover more distinct prompts per update under the same rollout
budget. A possible
explanation is that the full-parameter critic must learn prefix-level
values from sparse final rewards while tracking a changing policy.
Inaccurate or lagging value predictions can increase noise in the
resulting advantages. POISE instead fits a small readout to the actor's
value-relevant representations and uses an independent rollout to
construct each baseline. The separate critic also requires a model-scale set of parameters,
together with gradients and optimizer states during training. Keeping
these on GPU increases GPU memory requirements; offloading shifts much
of the requirement to host RAM and introduces data transfers. POISE
does not train an additional model-scale critic.
\clearpage

\clearpage
\section{Ablations}
\label{app:ablations}

\subsection{Controlled Setting for End-to-End Ablations}
\label{app:controlled_rl_ablations}

The end-to-end ablations in
Section~\ref{sec:estimator-analysis} are conducted separately from our main
multi-domain experiments. We use Qwen3-4B as the policy model and train it on
the English subset of DAPO-Math-17K. This single-domain setting allows us to
isolate individual estimator design choices without variation from the domain
mixture used in the main experiments.

All variants use the same actor initialization, training data, optimization
settings, and total trajectory budget per update (see Table~\ref{tab:hyperparameters_poise}). When the number of rollouts
per prompt changes, the number of distinct prompts is adjusted accordingly so
that the total rollout budget remains fixed. We evaluate each trained policy
using Avg@32 on AMC23, AMC24, AIME24, AIME25, AIME26, HMMT25, and BRUMO25, and
report the average across the seven benchmarks.

\subsection{Ablation of Estimator Input Features}
\label{app:ablation_input_features}

We first isolate the information contributed by each estimator input. We use
the held-out value-prediction setting from
Section~\ref{sec:Estimating_Baseline_Values}: the policy and rollout data are
held fixed, and predictions are evaluated against empirical Avg@8 values on
held-out prompts. Each single-feature variant receives only the indicated
feature type.


All three actor-internal signals are predictive of value. Reasoning hidden
states provide the strongest individual feature, followed by token entropy,
while prompt hidden states remain informative but substantially weaker.
Response length performs worse than any of the internal signals, indicating
that the estimator is not simply exploiting a correlation between response
length and correctness. Combining all three internal signals gives the best
MAE and Pearson correlation, supporting their complementary use in the full
estimator.

\subsection{Ablations of Hidden-State Extraction}
\label{app:ablations_hidden_state}

We next vary the layer and pooling window used to extract prompt and reasoning
hidden states from Qwen3-4B. Unless otherwise specified, the estimator uses
StandardScaler followed by ridge regression, and all non-ablated components are
held fixed.

\paragraph{Layer index.}
We fix the pooling window to the final 10 tokens and sweep the transformer layer
used for both prompt and reasoning features.

\begin{table}[H]
\centering
\small
\caption{
Layer ablation for Qwen3-4B using last-10-token mean pooling.
}
\label{tab:qwen3-4b-layer-ablation}
\setlength{\tabcolsep}{5pt}
\renewcommand{\arraystretch}{1.05}
\begin{tabular}{ccc}
\toprule
Layer & MAE $\downarrow$ & Pearson $r$ $\uparrow$ \\
\midrule
1  & 0.128 & 0.807 \\
3  & 0.126 & 0.821 \\
5  & 0.129 & 0.816 \\
7  & 0.129 & 0.820 \\
9  & 0.131 & 0.814 \\
11 & 0.131 & 0.819 \\
13 & 0.130 & 0.823 \\
15 & 0.128 & 0.828 \\
17 & 0.124 & 0.831 \\
19 & 0.123 & \textbf{0.834} \\
21 & 0.128 & 0.830 \\
23 & 0.130 & 0.824 \\
25 & 0.132 & 0.817 \\
27 & 0.125 & 0.827 \\
29 & 0.125 & 0.830 \\
31 & 0.126 & 0.824 \\
33 & \textbf{0.120} & 0.831 \\
35 & 0.123 & 0.827 \\
\bottomrule
\end{tabular}
\end{table}

Performance is strong across a broad range of layers. Layer 19 gives the highest
Pearson correlation, while layer 33 gives the lowest MAE. We use layer 19 in
the main experiments because it provides the strongest correlation while
remaining close to the best MAE.

\paragraph{Pooling window.}
We fix the extraction layer to 19 and vary the number of final tokens included
in mean pooling.

\begin{table}[H]
\centering
\small
\caption{
Pooling-window ablation for Qwen3-4B at layer 19.
}
\label{tab:qwen3-4b-pooling-ablation}
\setlength{\tabcolsep}{5pt}
\renewcommand{\arraystretch}{1.08}
\begin{tabular}{lcc}
\toprule
Pooling window & MAE $\downarrow$ & Pearson $r$ $\uparrow$ \\
\midrule
Last 5  & 0.127 & 0.822 \\
Last 10 & \textbf{0.124} & \textbf{0.834} \\
Last 15 & \textbf{0.124} & 0.833 \\
\bottomrule
\end{tabular}
\end{table}

The last-10 and last-15 variants perform nearly identically, while last-5
pooling is slightly weaker. The estimator is therefore not sensitive to the
exact pooling window within this range. We use the final 10 tokens as the
default.

\subsection{Ablations of Probe Architecture}
\label{app:ablations_probe_design}

We compare the linear ridge estimator with MLP probes of different depths and
widths. All variants receive the same prompt hidden states, reasoning hidden
states, and entropy features extracted from Qwen3-4B at layer 19 with
last-10-token mean pooling. MLPs use ReLU activations, dropout, AdamW, and early
stopping on a validation split.

\begin{table}[H]
\centering
\small
\caption{
Probe-architecture ablation for Qwen3-4B. MLP names denote
depth $\times$ hidden width.
}
\label{tab:qwen3-4b-probe-architecture-ablation}
\setlength{\tabcolsep}{5pt}
\renewcommand{\arraystretch}{1.08}
\begin{tabular}{lcc}
\toprule
Probe architecture & MAE $\downarrow$ & Pearson $r$ $\uparrow$ \\
\midrule
Linear ridge        & 0.124 & \textbf{0.834} \\
MLP $1\times512$    & 0.154 & 0.778 \\
MLP $3\times512$    & 0.123 & 0.780 \\
MLP $5\times512$    & 0.128 & 0.765 \\
MLP $7\times512$    & 0.124 & 0.763 \\
MLP $9\times512$    & 0.133 & 0.814 \\
MLP $3\times1024$   & \textbf{0.117} & 0.801 \\
MLP $5\times1024$   & 0.128 & 0.754 \\
\bottomrule
\end{tabular}
\end{table}

The best MLP slightly improves MAE, but none improves Pearson correlation over
the linear ridge estimator. Because the linear probe is also substantially
cheaper to refit online and introduces fewer optimization hyperparameters, we
use it throughout POISE.

\subsection{End-to-End RL Ablations}
\label{app:rl_design_ablations}

\paragraph{Ablations on algorithm design.} We next test the central design choices in end-to-end RL. The aggregate results
are reported in Table~\ref{tab:rl-design-ablation}.

First, we test a version of the estimator without any trajectory-leve features. This variant uses only prompt hidden states and requires a single rollout per
prompt. Because the estimator does not use trajectory-dependent features, its
baseline is action-independent without cross-rollout estimation. Comparing this
variant with full POISE tests whether the information added by reasoning states
and entropy is sufficient to justify sampling a second rollout.

Second, we test a variant of the estimator withour cross-rollout design. This variant retains the full feature set but predicts each rollout's baseline
from its own internal states. It therefore removes the cross-rollout
construction and makes the baseline dependent on the response being optimized.
We include this variant only as a diagnostic: it intentionally violates the
action-independence condition discussed in
Section~\ref{sec:preliminaries}.


Full POISE outperforms both the prompt-only and same-rollout variants. The first
comparison shows that trajectory-level features provide sufficient additional
information to justify the second rollout. The second shows that these features
must be introduced through an independent rollout rather than directly from the trajectory being optimized. 

\paragraph{Ablations on estimator components.}

We further test the trajectory-level estimator inputs in end-to-end RL. The
entropy-only, reasoning-state-only, and full variants all use two rollouts per
prompt and the same cross-rollout construction. Their rollout allocation and
policy-optimization settings are therefore identical, isolating the effect of
the input feature set.

\begin{table*}[t]
\centering
\scriptsize
\caption{
End-to-end ablation of estimator inputs on Qwen3-4B. All estimator variants
use two rollouts per prompt and cross-rollout baseline estimation. We report
Avg@32 (\%).
}
\label{tab:component-rl-ablation}
\setlength{\tabcolsep}{3.5pt}
\renewcommand{\arraystretch}{1.06}
\begin{tabular}{lcccccccc}
\toprule
Method
& AMC23 & AMC24 & AIME24 & AIME25 & AIME26
& HMMT25 & BRUMO25 & Avg. \\
\midrule
Base
& 42.2 & 31.9 & 26.3 & 19.6 & 24.4 & 12.9 & 21.7 & 25.8 \\

Entropy only
& 83.5 & 56.8 & 44.3 & 38.4 & 45.1 & 22.3 & 33.9 & 46.3 \\

Reasoning states only
& 84.2 & \textbf{60.0} & 47.9 & \textbf{42.0}
& \textbf{46.9} & \textbf{24.4} & 35.5 & 48.7 \\

Full POISE
& \textbf{88.7} & 59.2 & \textbf{51.0} & 41.7
& 46.4 & \textbf{24.4} & \textbf{39.3} & \textbf{50.1} \\
\bottomrule
\end{tabular}
\end{table*}

As visible in Table~\ref{tab:component-rl-ablation}, reasoning hidden states provide the strongest individual feature set, consistent
with the held-out value-prediction results. The full estimator nevertheless
achieves the highest aggregate score, showing that prompt states and entropy
provide complementary information that improves policy optimization overall.

\clearpage

\end{document}